%% file: lpmln-system-0718.tex
\documentclass{new_tlp}
\usepackage{times}
\usepackage{helvet}
\usepackage{courier}
\usepackage{amsmath}
\usepackage{amssymb}
\usepackage{graphicx}
\usepackage{listings}
\usepackage{url}
\usepackage{multicol,lipsum}
\usepackage{wrapfig}
\usepackage{longtable}

\long\def\BOC#1\EOC{\message{(Commented text )}}
\long\def\BOCC#1\EOCC{\message{(Commented text )}}
\long\def\BOCCC#1\EOCCC{\message{(Commented text )}}
\long\def\optional#1{\empty}
\long\def\NB#1{}

\def\bi{\begin{itemize}}
\def\ii{\item}
\def\ei{\end{itemize}} 
\def\beq{\begin{equation}}
\def\eeq#1{\label{#1}\end{equation}}
\def\ba{\begin{array}}
\def\ea{\end{array}}
\def\j#1{\hbox{\it #1\/}}
\def\mi#1{\mathit{#1}}
\def\lpmln{{\rm LP}^{\rm{MLN}}}
\def\no{\j{not}}

\def\ar{\leftarrow}
\def\rar{\rightarrow}

\def\false{\hbox{\bf f}}
\def\true{\hbox{\bf t}}
\def\lpmln{\hbox{\rm LP}^{\rm{MLN}}}
\def\lpmln{{\rm LP}^{\rm{MLN}}}
\def\sm{\rm SM}

\def\ML{\mathbb{L}}

\def\proof{\noindent{\bf Proof}.\hspace{3mm}}
\def\qed{\quad \vrule height7.5pt width4.17pt depth0pt \medskip}

\newcommand{\argmin}{\mathop{\mathrm{argmin}}\limits}
\newcommand{\argmax}{\mathop{\mathrm{argmax}}\limits}

\newtheorem{prop}{Proposition}
\newtheorem{thm}{Theorem}

\newtheorem{lemma}{Lemma}
\newtheorem{example}{Example}

\newcommand{\la}{\textsc{lpmln2asp} }
\newcommand{\lm}{\textsc{lpmln2mln} }
\newcommand{\lp}{$\lpmln$ }
\newcommand{\al}{{\sc alchemy}}
\newcommand{\ro}{\textsc{rockit}}
\newcommand{\tu}{\textsc{tuffy}}

\newcommand{\cli}{\textsc{clingo} }

\newcommand{\btm}[4]{\makecell{#1 + #2 \\ {[#3/#4]}}}
\newcommand{\btc}[3]{\makecell{#1 \\ {[#2/#3]}}}
\newcommand{\btr}[2]{\makecell{#1 \\ {[#2]}}}

\newcommand{\bex}[1]{\begin{example}#1\end{example}}

\begin{document}

\title{Computing $\lpmln$ Using ASP and MLN Solvers}  

\author[Lee, Talsania \& Wang]{Joohyung Lee, Samidh Talsania, and Yi Wang \\
School of Computing, Informatics, and Decision Systems Engineering \\
 Arizona State University, Tempe, USA \\
 \email{\{joolee, stalsani, ywang485\}@asu.edu}
}

\maketitle

\begin{abstract}
$\lpmln$ is a recent addition to probabilistic logic programming languages. Its main idea is to overcome the rigid nature of the stable model semantics by assigning a weight to each rule in a way similar to Markov Logic is defined. We present two implementations of $\lpmln$, {\sc lpmln2asp} and {\sc lpmln2mln}. System {\sc lpmln2asp} translates $\lpmln$ programs into the input language of answer set solver {\sc clingo}, and using  weak constraints and stable model enumeration, it can compute most probable stable models as well as exact conditional and marginal probabilities. System {\sc lpmln2mln} translates $\lpmln$ programs into the input language of Markov Logic solvers, such as {\sc alchemy}, {\sc tuffy},  and {\sc rockit}, and allows for performing approximate probabilistic inference on $\lpmln$ programs.  
We also demonstrate the usefulness of the  $\lpmln$ systems for computing other languages, such as ProbLog and Pearl's Causal Models, that are shown to be translatable into $\lpmln$. 
\end{abstract}

%---------------------------------------------------------------
\section{Introduction}\label{sec:intro}
%---------------------------------------------------------------

$\lpmln$ is a simple extension of answer set programs with the concept of weighted rules, whose weight scheme is adopted from that of Markov Logic \cite{richardson06markov}. Like Markov Logic, not all $\lpmln$ rules have to be true but, roughly speaking, the more rules are true, the larger weight is assigned to the corresponding stable model.

It is shown that both answer set programs and Markov Logic can be easily embedded in $\lpmln$ (Theorems~1, 2 of the paper by \citeN{lee16weighted}).  The other direction is more interesting from a computational point of view because it provides ways to compute $\lpmln$ using existing implementations of ASP and MLN (Markov Logic Network) solvers. This paper further develops the translations in this direction, presents two implementations of $\lpmln$ based on them along with a few new theorems justifying the implementations, and illustrates the usefulness of the systems. 

It is shown by~\citeN{balai16ontherelationship} that $\lpmln$ programs can be turned into P-log programs. We present a similar but simpler translation that turns $\lpmln$ programs into answer set programs. However, for a non-ground $\lpmln$ program, the translation yields an answer set program that is unsafe, so a direct implementation of this idea has a drawback. 
Instead, we develop an implementation based on the translation by~\citeN{lee17lpmln} that turns (ground) $\lpmln$ programs into answer set programs containing weak constraints so that the most probable stable models of an $\lpmln$ program coincide with the optimal stable models of the translated ASP program. Going further, we show how to map the penalty of each stable model of the translated ASP program to the probability of each stable model of a (non-ground) $\lpmln$ program, and use this result to compute probabilistic queries for an $\lpmln$ program using an ASP solver {\sc clingo}.
The input language of {\sc lpmln2asp} is familiar to  the users of {\sc clingo} because its syntax is a simple extension of {\sc clingo} rules prepended by weights, thereby allowing many advanced constructs of the {\sc clingo} language, such as aggregates and conditional literals, in the context of $\lpmln$.

A different method to compute an $\lpmln$ program is by converting it into a Markov Logic Network \cite[Theorem~3]{lee16weighted}, similar to the reduction of answer set programs to propositional logic, and then invoking MLN solvers, such as {\sc alchemy}, {\sc tuffy}, and {\sc rockit}. While it is possible to turn any $\lpmln$ program into an equivalent MLN by adding all loop formulas, in practice, the straightforward implementation does not yield an effective computation. Thus, we limit attention to the ``tight'' fragment of $\lpmln$ programs that can be easily converted into MLNs using the process of completion.
Even so, the straightforward encoding of completion formulas in Markov Logic may lead to a blow-up in CNF conversion performed by {\sc alchemy} because the conversion is naively implemented in {\sc alchemy}. Furthermore, the input languages of {\sc tuffy} and {\sc rockit} do not even allow nested formulas, which are needed to encode completion formulas. So, {\sc lpmln2mln} implements some equivalent transformation using auxiliary atoms to avoid the blow-up in CNF conversion and takes care of differences in the input language of different MLN solvers.
The input language of {\sc lpmln2mln} resembles that of {\sc alchemy} and is converted into one of the input languages of {\sc alchemy}, {\sc tuffy}, and {\sc rockit} depending on the mode selected. 
The system utilizes approximate probabilistic inference methods or exact optimization methods supported by the MLN solvers.

The implementations are not only interesting for computing $\lpmln$. System {\sc lpmln2asp} can be used to derive the most probable stable models even when the standard answer set program is inconsistent. This feature could be useful in debugging an inconsistent answer set program or deriving some meaningful conclusions from an inconsistent knowledge base. Also, both implementations can be used to compute  other formalisms,  such as Markov Logic, ProbLog \cite{deraedt07problog}, Pearl's Probabilistic Causal Models \cite{pearl00causality}, and P-log \cite{baral09probabilistic}, which are shown to be translatable into $\lpmln$~\cite{lee16weighted,lee15markov,lee17lpmln}.

The systems are publicly available at
%\begin{center}
\url{http://reasoning.eas.asu.edu/lpmln/}
%\end{center}
 along with the user manual and examples.  

The paper is organized as follows. 
Section~\ref{sec:prelim} reviews the language $\lpmln$, which is based on the concept of reward, and presents a reformulation of $\lpmln$ based on the concept of penalty. 
Section~\ref{sec:lpmln2asp} shows two translations of $\lpmln$ programs into answer set programs, one based on the reward-based way and another based on the penalty-based way, and presents system {\sc lpmln2asp} that implements the penalty-based way. 
Section~\ref{sec:lpmln2mln} shows a translation of tight $\lpmln$ programs into Markov Logic Networks and presents system {\sc lpmln2mln} that implements the translation.
Section~\ref{sec:example} gives a comparison and running statistics of these implementations. 
Section~\ref{sec:other} shows how to use these systems to compute other probabilistic logic languages that are shown to be translatable into $\lpmln$.
The proofs of the theorems and more experiments can be found in the online appendix accompanying the paper at the TPLP archive \cite{lee17computing-online}.

%--------------------------------------------------------
\section{Language $\lpmln$} \label{sec:prelim}
%--------------------------------------------------------

%------------------------------------------------------------------
\subsection{Review: $\lpmln$} \label{ssec:lpmln}
%------------------------------------------------------------------

We review the definition of $\lpmln$ from the paper by~\citeN{lee16weighted}.
An $\lpmln$ program is a finite set of weighted rules $w: R$ where $R$ is a rule (as allowed in the input language of ASP solver {\sc clingo}), and   $w$ is a real number (in which case, the weighted rule is called {\em soft}) or $\alpha$ for denoting the infinite weight (in which case, the weighted rule is called {\em hard}). An $\lpmln$ program is called {\em ground} if its rules contain no variables. We assume a finite Herbrand Universe so that the ground program is finite.
Each ground instance of a non-ground rule receives the same weight as the original non-ground formula.

For any ground $\lpmln$ program $\Pi$ and any interpretation~$I$, 
$\overline{\Pi}$ denotes the usual (unweighted) ASP program obtained from $\Pi$ by dropping the weights, and
${\Pi}_I$ denotes the set of $w: R$ in $\Pi$ such that $I\models R$, 
and $\sm[\Pi]$ denotes the set $\{I \mid \text{$I$ is a stable model of $\overline{\Pi_I}$}\}$.
The {\em unnormalized weight} of an interpretation $I$ under $\Pi$ is defined as 
\[
%\small
 W_\Pi(I) =
\begin{cases}
  exp\Bigg(\sum\limits_{w:R\;\in\; {\Pi}_I} w\Bigg) & 
      \text{if $I\in\sm[\Pi]$}; \\
  0 & \text{otherwise}. 
\end{cases}
\]
The {\em normalized weight} (a.k.a. {\em probability}) of an interpretation $I$ under~$\Pi$ is defined as  \[ 
\small 
  P_\Pi(I)  = 
  \lim\limits_{\alpha\to\infty} \frac{W_\Pi(I)}{\sum\limits_{J\in {\rm SM}[\Pi]}{W_\Pi(J)}}. 
\] 
Interpretation $I$ is called a {\sl (probabilistic) stable model} of $\Pi$ if $P_\Pi(I)\ne 0$. The most probable stable models of $\Pi$ are the stable models with the highest probability.

%------------------------------------------------------------------
\subsection{Reformulating $\lpmln$ Based on the Concept of Penalty}\label{ssec:reformulation}
%------------------------------------------------------------------

In the definition of the $\lpmln$ semantics by \citeN{lee16weighted}, the weight assigned to each stable model can be regarded as ``rewards": the more rules are true in deriving the stable model, the larger weight is assigned to it. 
In this section, we reformulate the $\lpmln$ semantics in a ``penalty'' based way.
More precisely, the penalty based weight of an interpretation $I$ is defined as the exponentiated negative sum of the weights of the rules that are not satisfied by $I$ (when $I$ is a stable model of $\overline{\Pi_I}$).  
Let 
\[
%\small
 W^{\rm pnt}_\Pi(I) =
\begin{cases}
  exp\Bigg(-\sum\limits_{w:R\;\in\; {\Pi} \text{ and } I\not\models R} w\Bigg) & 
      \text{if $I\in\sm[\Pi]$}; \\
  0 & \text{otherwise} 
\end{cases}
\]
and
$$ 
  P^{\rm pnt}_\Pi(I) = 
    \lim\limits_{\alpha\to\infty} \frac{W^{\rm pnt}_\Pi(I)}{\sum\limits_{J\in {\rm SM}[\Pi]}{W^{\rm pnt}_\Pi(J)}}.
$$ 

The following theorem tells us that the $\lpmln$ semantics can be reformulated using the concept of a penalty-based weight. 

\begin{thm}\label{thm:lpmln-pnt}\optional{thm:lpmln-pnt}
For any $\lpmln$ program $\Pi$ and any interpretation $I$, 
\[
  W_{\Pi}(I) \propto W^{\rm pnt}_{\Pi}(I)% \times TW_{\Pi}
\text{ \ \ \ \ and \ \ \ \ }
  P_\Pi(I) = P_\Pi^{\rm pnt}(I).
\]
\end{thm}

Although the penalty-based reformulation appears to be more complicated, it has a few desirable features. One of them is that adding a trivial rule does not affect the weight of an interpretation,  which is not the case with the original definition. More importantly, this reformulation leads to a better translation of $\lpmln$ programs into answer set programs as we discuss in Section~\ref{ssec:lpmln2asp-pnt}.

%--------------------------------------------------------
\section{Turning $\lpmln$ into ASP with Weak Constraints}\label{sec:lpmln2asp}
%--------------------------------------------------------

%------------------------------------------------------------------
\subsection{Review: Weak Constraints} \label{ssec:weak} 
%------------------------------------------------------------------

A {\em weak constraint} \cite{buccafurri00enhancing,calimeri12aspcore2} has the form 
\[
  {\tt :\sim}\ F \ \ \  [\j{Weight}\ @\ \j{Level}] 
\] %eeq{wc}
where $F$ is a conjunction of literals, $\j{Weight}$ is a real number, and $\j{Level}$ is a nonnegative integer. 

Let $\Pi$ be a program $\Pi_1\cup\Pi_2$, where $\Pi_1$ is an answer set program that does not contain weak constraints, and $\Pi_2$ is a set of ground weak constraints. We call $I$ a stable model of $\Pi$ if it is a stable model of $\Pi_1$.
For every stable model $I$ of $\Pi$ and any nonnegative integer $l$, the {\em penalty} of $I$ at level~$l$, denoted by $\j{Penalty}_\Pi(I,l)$,  is defined as 
\[
\sum\limits_{:\sim\ F [w@l]\in\Pi_2,\atop I\models F}  w .
\]
For any two stable models $I$ and $I'$ of~$\Pi$, we say $I$ is {\em dominated} by $I'$ if 
\bi
\ii there is some nonnegative integer $l$ such that $\j{Penalty}_\Pi(I',l) < \j{Penalty}_\Pi(I,l)$ and 
\ii for all integers $k>l$, $\j{Penalty}_\Pi(I',k) = \j{Penalty}_\Pi(I,k)$.
\ei
A stable model of $\Pi$ is called {\em optimal} if it is not dominated by another stable model of $\Pi$. 

The input language of {\sc clingo} allows non-ground weak constraints that contain tuples of terms. 

%------------------------------------------------------------------
\subsection{Turning $\lpmln$ into ASP: Reward Way}\label{ssec:lpmln2asp-rwd}
%------------------------------------------------------------------

In the paper by~\citeN{balai16ontherelationship}, it is shown that $\lpmln$ programs can be turned into P-log. In this section, we show that using a similar translation, it is even possible to turn $\lpmln$ programs into  answer set programs. 

We turn each (possibly non-ground) rule
\[
   w_i:\ \ \ \  \j{Head}_i({\bf x}) \leftarrow \j{Body}_i({\bf x}) \  
\] %eeq{lpmln2asp-rule}
%($i$ is the index of the rule)
in an $\lpmln$ program $\Pi$, where $i$ is the index of the rule and ${\bf x}$ is the list of global variables in the rule, 
into ASP rules
\beq
\ba {rcl}
  {\tt sat}(i, w_i, {\bf x}) & \ar & Head_i({\bf x})\\
  {\tt sat}(i, w_i, {\bf x}) & \ar & {\tt not}\ Body_i({\bf x})\\
    Head_i({\bf x})  & \ar &  Body_i({\bf x}),\ {\tt not}\ {\tt not}\ {\tt sat}(i, w_i, {\bf x})\\
           &:\sim & {\tt sat}(i, w_i, {\bf x}). \ \ \ [-w'_i@l, i, {\bf x}]
\ea
\eeq{lpmln2asp-rwd}
where (i) $w'_i = 1$ and $l=1$ if $w_i$ is $\alpha$; and (ii) $w'_i = w_i$ and $l=0$ otherwise.\footnote{{\sc clingo} restricts the weights in weak constraints to be integers only. To implement the translation using {\sc clingo}, we need to turn $w'_i$ into an integer by multiplying some factor.} 

Intuitively, a ground ${\tt sat}$ atom is true if the corresponding ground rule obtained from the original program is true. For each true ${\tt sat}$ atom, a weak constraint imposes on the stable model the opposite of the weight as a penalty, which can be viewed as imposing the weight as a reward.

By ${\sf lpmln2asp^{rwd}}(\Pi)$ we denote the resulting ASP program containing weak constraints. The following theorem states the correctness of the translation. Let $Gr(\Pi)$ be the ground program obtained from $\Pi$ by replacing global variables with the Herbrand Universe.

\begin{thm}\label{thm:lpmln2asp-rwd}\optional{thm:lpmln2asp-rwd}
For any $\lpmln$ program $\Pi$, there is a 1-1 correspondence $\phi$ between $\sm[\Pi]$ and the set of stable models of ${\sf lpmln2asp^{rwd}}(\Pi)$, where 
\[ 
  \phi(I)=I\cup \{{\tt sat}(i, w_i, {\bf c}) \mid 
                   w_i:\j{Head}_i({\bf c}) \ar \j{Body}_i({\bf c})\ \text{in ${Gr}(\Pi)$}, 
                   I\models \j{Body}_i({\bf c})\rar\j{Head}_i({\bf c})\}.
\]
Furthermore,
\begin{equation}\label{eq:wgt-sat}
   W_{\Pi}(I)= exp \Bigg(\sum_{{\tt sat}(i, w_i, {\bf c}) \in \phi(I)} w_i\Bigg).
\end{equation}
Also, $\phi$ is a 1-1 correspondence between the most probable stable models of $\Pi$ and the optimal stable models of ${\sf lpmln2asp^{rwd}}(\Pi)$.
\end{thm}

While the translation is simple and modular, there are a few problems with using this translation to compute $\lpmln$ using ASP solvers. First, the translation does not  necessarily yield a program that is acceptable in {\sc clingo} and requires a further translation. In particular, the first and the second rules of \eqref{lpmln2asp-rwd} may not be in the syntax of {\sc clingo}. (The third rule contains double negations, which are allowed in {\sc clingo} from version~4.)
Second, more importantly, when we translate non-ground $\lpmln$ rules into the input language of ASP solvers,  the first and the second rules of \eqref{lpmln2asp-rwd} may be unsafe, so {\sc clingo} cannot ground the program. 
An alternative translation in the next section avoids these problems by basing on the penalty-based concept of weights. 

%------------------------------------------------------------------
\subsection{Turning $\lpmln$ into ASP: Penalty Way} \label{ssec:lpmln2asp-pnt}
%------------------------------------------------------------------

Based on the reformulation of $\lpmln$ in Section~\ref{ssec:reformulation}, we introduce another translation that turns $\lpmln$ programs into ASP programs. The translation ensures that a safe $\lpmln$ program is always turned into a safe ASP program, and the resulting program  is readily acceptable as an input to {\sc clingo}.\footnote{An $\lpmln$ program $\Pi$ is {\em safe} if its unweighted program $\overline{\Pi}$ is safe as defined by~\citeN{calimeri12aspcore2}.}

We define the translation~${\sf lpmln2asp^{pnt}}(\Pi)$ by translating each (possibly non-ground) rule 
%We translate each ground rule of the form
\[
   w_i:\ \ \ \  \j{Head}_i({\bf x}) \leftarrow \j{Body}_i({\bf x}) \  
\] %eeq{lpmln2asp-rule}
in an $\lpmln$ program $\Pi$, where $i$ is the index of the rule and ${\bf x}$ is the list of global variables in the rule,  into ASP rules
\beq
\ba {rcl}
  {\tt unsat}(i, w_i, {\bf x}) & \ar & Body_i({\bf x}),\ {\tt not}\ Head_i({\bf x})\\
  Head_i({\bf x})  & \ar &  Body_i({\bf x}),\ {\tt not}\ {\tt unsat}(i, w_i, {\bf x})\\
           &:\sim & {\tt unsat}(i, w_i, {\bf x}). \ \ \ [w'_i@l, i, {\bf x}]
\ea
\eeq{lpmln2asp-pnt}
where (i) $w'_i = 1$ and $l=1$ if $w_i$ is $\alpha$; and (ii) $w'_i = w_i$ and $l=0$ otherwise.\footnote{In the case $\j{Head}_i$ is a disjunction $l_1; \dots, l_n$, expression ${\tt not}\ \j{Head}_i$ stands for ${\tt not}\ l_1, \dots, {\tt not}\ l_n$.}

Intuitively, the first rule of \eqref{lpmln2asp-pnt} makes atom ${\tt unsat}(i,w_i, {\bf x})$ true when the $i$-th rule in the original program is not satisfied. In that case, the second rule is not effective, and $w_i$ is imposed on the penalty of the stable model. On the other hand, if the $i$-th rule is satisfied, atom ${\tt unsat}(i,w_i,{\bf x})$ is false, the rule $\j{Head}_i\ar\j{Body}_i$ is effective, and the penalty is not imposed. 

The following theorem is an extension of Corollary~2 by~\citeN{lee17lpmln} to allow non-ground programs and to consider the correspondence between all stable models, not only the most probable ones.

\begin{thm}\label{thm:lpmln2asp-penalty}
For any $\lpmln$ program $\Pi$, there is a 1-1 correspondence $\phi$ between $\sm[\Pi]$ and the set of stable models of ${\sf lpmln2asp^{pnt}}(\Pi)$, 
where 
\[ 
  \phi(I)=I\cup \{{\tt unsat}(i, w_i, {\bf c}) \mid 
                   w_i:\j{Head}_i({\bf c}) \ar \j{Body}_i({\bf c})\ \text{in ${Gr}(\Pi)$}, 
                   I\not\models \j{Body}_i({\bf c})\rar\j{Head}_i({\bf c})\}.
\]
Furthermore,
\begin{equation}\label{eq:wgt-unsat}
   W^{\rm pnt}_{\Pi}(I)= exp \Bigg(-\sum_{{\tt unsat}(i, w_i, {\bf c}) \in \phi(I)} w_i\Bigg). 
\end{equation}
Also, $\phi$ is a 1-1 correspondence between the most probable stable models of $\Pi$ and the optimal stable models of ${\sf lpmln2asp^{pnt}}(\Pi)$.
\end{thm}

Theorem~\ref{thm:lpmln2asp-penalty}, in conjunction with Theorem~\ref{thm:lpmln-pnt}, provides a way to compute the probability of a stable model of an $\lpmln$ program by examining the ${\tt unsat}$ atoms satisfied by the corresponding stable model of the translated ASP program.

%----------------------------------------------------------------------------
\subsection{System {\sc lpmln2asp}} \label{ssec:lpmln2asp-system}
%----------------------------------------------------------------------------

\begin{figure}
    \includegraphics[width=0.85\textwidth,height=3.4cm]{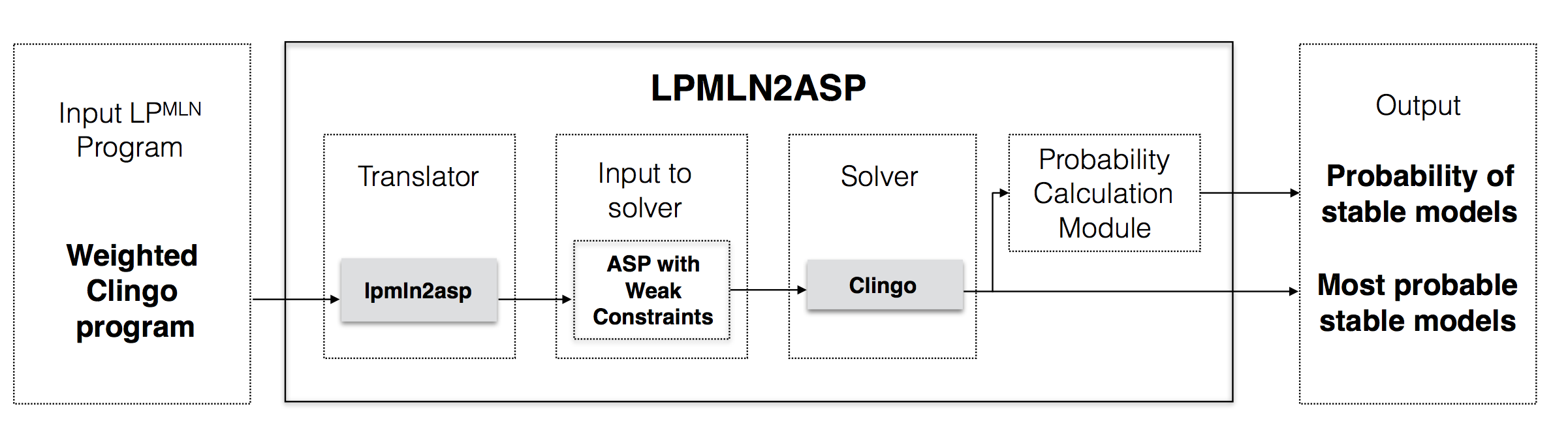}
    \caption{Architecture of System {\sc lpmln2asp}}
    \label{fig:lpmln2asp-sys}
\end{figure}

System {\sc lpmln2asp} is an implementation of $\lpmln$ based on the result in Section~\ref{ssec:lpmln2asp-pnt} using {\sc clingo} v4.5.
It can be used for computing the probabilities of stable models, marginal/conditional probability of a query, as well as the most probable stable models. 

In the input language of {\sc lpmln2asp}, a soft rule is written in the form
\beq
   w_i\ \ \j{Head}_i\ar\j{Body}_i
\eeq{lpmln2asp-rule}
where $w_i$ is a real number in decimal notation, and $\j{Head}_i\ar\j{Body}_i$ is a {\sc clingo} rule. A hard rule is written without weights and is identical to a {\sc clingo} rule. 
For instance, the ``Bird'' example from the paper by~\citeN{lee16weighted} can be represented in the input language of {\sc lpmln2asp} as follows. The first three rules represent definite knowledge while the last two rules represent uncertain knowledge with different confidence levels. 
\begin{lstlisting}
  % bird.lpmln
  bird(X) :- residentbird(X).
  bird(X) :- migratorybird(X).
  :- residentbird(X), migratorybird(X).
  2 residentbird(jo).
  1 migratorybird(jo).
\end{lstlisting}

The basic command line syntax of executing {\sc lpmln2asp} is
\begin{lstlisting}
   lpmln2asp -i <input file> [-r <output file>] [-e <evidence file>]  
         [-q <query predicates>] [-hr] [-all] [-clingo "<clingo options>"] 
\end{lstlisting}
which follows the {\sc alchemy} command line syntax.

The mode of computation is determined by the options provided to {\sc lpmln2asp}. 
By default, the system finds a most probable stable model of ${\sf lpmln2asp}^{\rm pnt}(\Pi)$ (MAP estimate) by leveraging {\sc clingo}'s built-in optimization method for weak constraints.  

For computing marginal probability, {\sc lpmln2asp} utilizes {\sc clingo}'s interface with Python. When {\sc clingo} enumerates each stable model of ${\sf lpmln2asp}^{\rm pnt}(\Pi)$, the computation is interrupted by the {\em probability computation module}, a Python program which records the stable model as well as its penalty specified in the ${\tt unsat}$ atoms true in the stable model. Once all the stable models are generated, the control returns to the module, which sums up the recorded penalties to compute the normalization constant as well as the probability of each stable model. The probabilities of query atoms (specified by the option {\tt -q}) are also calculated by adding the probabilities of the stable models that contain the query atoms. For instance, the probability of a query atom {\tt residentbird(jo)} is $\Sigma_{I\models {\tt residentbird(jo)}} P(I)$.
The option {\tt -all} instructs the system to display all stable models and their probabilities.

For conditional probability, the evidence file {\tt <evidence file>} is specified by the option {\tt -e}. The file may contain any {\sc clingo} rules, but usually they are constraints, i.e., rules with the empty head. The main difference from the marginal probability computation is that {\sc clingo} computes ${\sf lpmln2asp}^{\rm pnt}(\Pi) \cup$ {\tt <evidence file>} instead of ${\sf lpmln2asp}^{\rm pnt}(\Pi)$.

Below we illustrate how to use the system for various inferences.

\medskip\noindent
{\bf MAP (Maximum A Posteriori) inference:}\ \ 
The command line to use is %\lstinline{lpmln2asp -i <input_file>} 
\begin{lstlisting}
    lpmln2asp -i <input file> 
\end{lstlisting}
By default, {\sc lpmln2asp} computes MAP inference. For example, 
\lstinline{lpmln2asp -i bird.lpmln} returns 
\begin{lstlisting}
residentbird(jo) bird(jo) unsat(5,"1.000000")
Optimization: 1000
OPTIMUM FOUND
\end{lstlisting}

%-----------------------------------------------
\smallskip\noindent
{\bf Marginal probability of all stable models:}\ \ 
The command line to use is 
\begin{lstlisting}
    lpmln2asp -i <input file> -all 
\end{lstlisting}
%Appending \lstinline|-all| argument invokde the proabbility computation module in {\sc lpmln2asp} and also serves as verbose mode to list all models and theire respective proababilities. 
For example, \lstinline{lpmln2asp -i bird.lpmln -all}  outputs 
%\begin{multicols}{1}
\begin{lstlisting}
Answer: 1
residentbird(jo) bird(jo) 
unsat(5,"1.000000")
Optimization: 1000
Answer: 2
unsat(4,"2.000000") unsat(5,"1.000000")
Optimization: 3000
Answer: 3
unsat(4,"2.000000") bird(jo) 
migratorybird(jo)
Optimization: 2000

Probability of Answer 1 : 0.665240955775
Probability of Answer 2 : 0.0900305731704
Probability of Answer 3 : 0.244728471055
\end{lstlisting}
%\end{multicols}

%-----------------------------------------------
\noindent
{\bf Marginal probability of query atoms:}\ \ 
The command line to use is 
\begin{lstlisting}
    lpmln2asp -i <input file> -q <query predicates>
\end{lstlisting}  
This mode calculates the marginal probability of the atoms whose predicates are specified by {\tt -q} option. For example, 
\lstinline{lpmln2asp -i birds.lp -q residentbird} outputs
\begin{lstlisting}
residentbird(jo) 0.665240955775
\end{lstlisting}

%-----------------------------------------------
\smallskip\noindent
{\bf Conditional probability of query given evidence:}\ \  
The command line to use is 
\begin{lstlisting}
    lpmln2asp -i <input file> -q <query predicates> -e <evidence file>
\end{lstlisting}
This mode computes the conditional probability of a query given the evidence specified in the \lstinline{<evidence file>}. 
For example,
\begin{lstlisting}
  lpmln2asp -i birds.lp -q residentbird -e evid.db
\end{lstlisting}
where \lstinline{evid.db} contains 
\begin{lstlisting}
   :- not bird(jo).
\end{lstlisting}
outputs the conditional probability $P(residentbird(X) \mid bird(jo))$:
\begin{lstlisting}
   residentbird(jo) 0.73105857863
\end{lstlisting}

%-----------------------------------------------

\medskip\noindent
{\bf Debugging ASP Programs:}\ \ 
The command line to use is 
\begin{lstlisting}
   lpmln2asp -i <input file> -hr -all
\end{lstlisting}
By default, {\sc lpmln2asp} does not translate hard rules and pass them to {\sc clingo} as is. The option {\tt -hr} instructs the system to translate hard rules as well.
 According to Proposition~2 by \citeN{lee16weighted}, as long as the $\lpmln$ program has a probabilistic stable model that satisfies all hard rules, the simpler translation that does not translate hard rules gives the same result as the full translation and is more computationally efficient. Since in many cases hard rules represent definite knowledge that should not be violated, this is desirable. 

On the other hand, translating hard rules could be relevant in some other cases, such as debugging an answer set program by finding which rules cause inconsistency. 
For example,  consider a {\sc clingo} input program {\tt bird.lp}, that is similar to {\tt bird.lpmln} but drops the weights in the last two rules. {\sc clingo} finds no stable models for this program. However, if we invoke {\sc lpmln2asp} on the same program as
\begin{lstlisting}
   lpmln2asp -i bird.lp -hr 
\end{lstlisting} 
the output of {\sc lpmln2asp} shows three probabilistic stable models, each of which shows a way to resolve the inconsistency by ignoring the minimal number of the rules. For instance, one of them is \{{\tt bird(jo)}, {\tt residentbird(jo)}\}, which disregards the last rule. 
The other two are similar.

Note that the probability computation involves enumerating all stable models so that it can be much more computationally expensive than the default MAP  inference. On the other hand, the computation is exact, so compared to an approximate inference, the ``gold standard'' result is easy to understand. Also, the conditional probability is more effectively computed than the marginal probability because {\sc clingo} effectively prunes many answer sets that do not satisfy the constraints specified in the evidence file. 

\subsection{Computing MLN with {\sc lpmln2asp}}\label{ssec:mln2asp}

A typical example in the MLN literature is a social network domain that describes how smokers influence other people, which can be represented in $\lpmln$ as follows. We assume three people $alice$, $bob$, and $carol$, and assume that $alice$ is a smoker, $alice$ influences $bob$, $bob$ influences $carol$, and nothing else is known.   
\beq
\ba l
  w:\ \   smoke(x) \wedge influence(x, y) \rightarrow smoke(y) \\
  \alpha:  smoke(alice) \hspace{1cm} \alpha: influence(alice, bob) \hspace{1cm}
  \alpha: influence(bob, carol).
 \ea
\eeq{smoke}
($w$ is a positive number.)
One may expect $bob$ is less likely a smoker than $alice$, and $carol$ is less likely a smoker than $bob$.

Indeed, the program above defines the following distribution (we omit the $\j{influence}$ relation, which has a fixed interpretation.)
\begin{center}
\begin{tabular}{ c c } 
 \hline
 \bf{Possible World} & \bf{Weight}\\ 
 \hline
 $\{smoke(alice), \neg smoke(bob), \neg smoke(carol)\}$ & $k\cdot e^{8w}$ \\ 
$\{smoke(alice), smoke(bob),\neg smoke(carol)\}$ & $k\cdot e^{8w}$ \\ 
$\{smoke(bob), \neg smoke(alice), smoke(carol)\}$ & $0$  \\ 
$\{smoke(alice), smoke(bob), smoke(carol)\}$ & $k\cdot e^{9w}$  \\ 
 \hline
\end{tabular}
\end{center}
where $k=e^{3\alpha}$. 
The normalization constant is the sum of all the weights: $k\cdot e^{9w}+2k\cdot e^{8w}$.
This means
 $P(smoke(alice)) = 1$
and 
\[
{\small
P(smoke(bob)) =\lim_{\alpha\to\infty} \frac{k\cdot e^{8w}+k\cdot e^{9w}}{k\cdot e^{9w}+2k\cdot e^{8w}} > 
    P(smoke(carol))= \lim_{\alpha\to\infty}\frac{k\cdot e^{9w}}{k\cdot e^{9w}+2k\cdot e^{8w}} .
}
\]

The result can be verified by {\sc lpmln2asp}.
For $w=1$, the input program {\tt smoke.lpmln} is
\begin{lstlisting}
   1 smoke(Y) :- smoke(X), influence(X, Y).
   smoke(alice).    influence(alice, bob).    influence(bob, carol).
\end{lstlisting}

Executing
%\begin{lstlisting}
{\tt    lpmln2asp -i smoke.lpmln -q smoke}
%\end{lstlisting}
outputs
\begin{lstlisting}
   smoke(alice) 1.00000000000000
   smoke(bob) 0.788058442382915
   smoke(carol) 0.576116884765829
\end{lstlisting}
as expected. 

On the other hand, if \eqref{smoke} is understood under the MLN semantics (assuming {\tt influence} relation is fixed as before),
similar to above, one can compute 
\[
  P(smoke(bob)) = \frac{e^{8w}+e^{9w}}{3e^{8w}+e^{9w}} = P(smoke(carol)).
\]
In other words, the degraded probability along the transitive relation does not hold under the MLN semantics. This is related to the fact that Markov logic cannot express the concept of transitive closure correctly as it inherits the FOL semantics. 

According to Theorem 2 in the paper by \citeN{lee16weighted}, MLN can be easily embedded in $\lpmln$ by adding a choice rule for each atom with an arbitrary weight, similar to the way propositional logic can be embedded in ASP using choice rules. 
Consequently, it is possible to use system {\sc lpmln2asp} to compute MLN, which is essentially using an ASP solver to compute MLN.

Let {\tt smoke.mln} be the resulting program. Executing
%\begin{lstlisting}
{\tt lpmln2asp -i smoke.mln -q smoke}
%\end{lstlisting}
outputs
\begin{lstlisting}
   smoke(alice) 1.0    smoke(bob) 0.650244590946    smoke(carol) 0.650244590946
\end{lstlisting}
which agrees with the computation above.

%--------------------------------------------------------
\section{Turning Tight $\lpmln$ into MLN} \label{sec:lpmln2mln}
%--------------------------------------------------------

%--------------------------------------------------------
\subsection{Translation}\label{ssec:lpmln2mln}
%--------------------------------------------------------

In the implementation of $\lpmln$ using MLN solvers, we limit attention to non-disjunctive logic programs that are tight.
Extending Theorem~3 by \citeN{lee16weighted} to non-ground programs, $\lpmln$ programs (possibly containing variables) can be turned into MLN instances by first rewriting each rule into Clark normal form 
\[
   w:\ \ p({\bf x}) \ar \j{Body}
\]
using equality in the body, where ${\bf x}$ is a list of new distinct variables unique to each predicate $p$, and then adding the completion formulas 
\beq
  \alpha:\ \  
     p({\bf x}) \rar\bigvee_{w:\ \ p({\bf x})\ar \mi{Body}\ \ \in\ \  \Pi} \j{Body} 
\eeq{completion}
for each atom $p({\bf x})$. 
In fact, since the built-in algorithm in {\sc alchemy} for clausifying the completion formulas may yield an exponential blow-up,  {\sc lpmln2mln} implements an equivalent rewriting known as Tseytin transformation,\footnote{%
\url{https://en.wikipedia.org/wiki/Tseytin_transformation}}
which introduces an auxiliary predicate for each disjunctive term in \eqref{completion}. 
The resulting MLN instance (possibly containing variables) is fed into {\sc alchemy}, which grounds the MLN and performs probabilistic inference on the ground network.\footnote{Without introducing auxiliary atoms in the process of clausification, most examples cannot be run on {\sc alchemy} because the CNF conversion method implemented in {\sc alchemy} is naive.}

For any MLN $\ML$, let $\ML^{hard}$ and $\ML^{soft}$ denote the set of hard formulas and soft formulas in $\ML$, respectively. For any set $\ML$ of weighted formulas, let $\overline{\ML}$ be the formulas obtained from $\ML$ by dropping the weights. The following proposition justifies the equivalent rewriting using auxiliary atoms. 
\begin{prop}\label{prop:mln-aux}\optional{prop:mln-aux}
For any MLN $\ML$ of signature $\sigma$, let $F({\bf x})$ be a subformula of some formula in $\ML$ where ${\bf x}$ is the list of all free variables of $F({\bf x})$, and let $\ML^F_{Aux}$ be the MLN program obtained from $\ML$ by replacing $F({\bf x})$ with a new predicate $Aux({\bf x})$ and adding the formula
\[
  \alpha\ \ :\ \  Aux({\bf x}) \leftrightarrow F({\bf x}).
\]
For any interpretation $I$ of $\ML$, let $I_{Aux}$ be the extension of $I$ of signature $\sigma\cup \{Aux\}$ defined by $I_{Aux}(Aux({\bf c}))=(F({\bf c}))^I$ for every list ${\bf c}$ of elements in the Herbrand universe.
When $\overline{\ML^{hard}}$ has at least one model, we have 
\[
    P_{\ML}(I) = P_{\ML^F_{Aux}}(I_{Aux}).
\]
\end{prop}

%-------------------------------------------------------
\subsection{System {\sc lpmln2mln}}
%--------------------------------------------------------

System {\sc lpmln2asp} is an implementation of $\lpmln$ based on the result in Section~\ref{ssec:lpmln2mln} using {\sc alchemy} (v2.0), {\sc tuffy} (v0.3) and {\sc rockit} (v0.5).

\begin{figure}
    \includegraphics[width=0.85\textwidth,height=3.5cm]{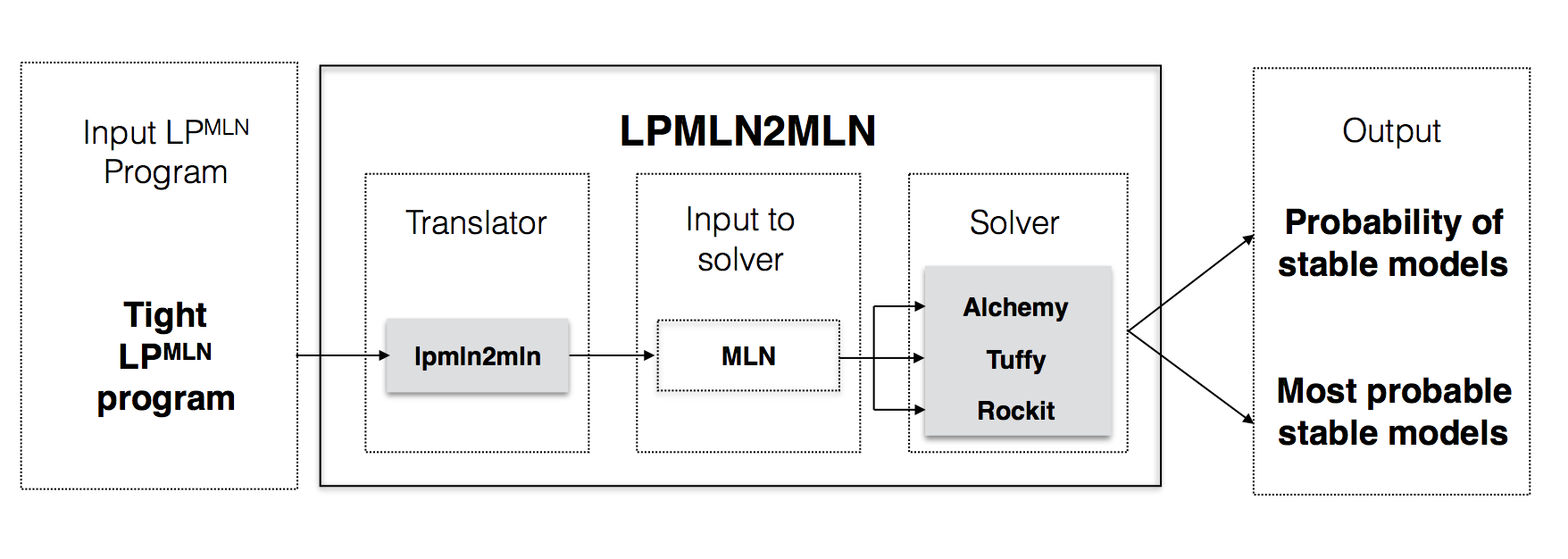}
    \caption{Architecture of System {\sc lpmln2mln}}
    \label{fig:lpmln2mln-sys}
\end{figure}

The basic command line syntax of executing {\sc lpmln2mln} is 
\begin{lstlisting}
   lpmln2mln -i <input file> -r <output file> -q <query predicates> 
        [-e <evidence file>] 
        [-tuffy| -rockit| -alchemy] [-mln "<options for mln solvers>"] 
\end{lstlisting}
which is similar to the command of executing {\sc lpmln2asp}.
%It translates the $\lpmln$ program \texttt{<input file>} into an MLN program and performs inference.

The syntax of the input language of {\sc lpmln2mln} follows that of {\sc alchemy}, except that it uses a rule form.
For example, consider again Example 1 in the paper by  \citeN{lee16weighted}. In the input language of {\sc lpmln2mln}, it is encoded as 

\begin{multicols}{2}
\begin{lstlisting}
   entity={Jo}

   Bird(entity)
   MigratoryBird(entity)
   ResidentBird(entity)

   Bird(x) <= ResidentBird(x).
   Bird(x) <= MigratoryBird(x).
   <= ResidentBird(x) ^ MigratoryBird(x).

   2 ResidentBird(Jo)
   1 MigratoryBird(Jo)
\end{lstlisting}
\end{multicols}
\vspace{-5mm}

Executing
\begin{lstlisting}
  lpmln2mln -i bird.lpmln -r out -q Bird,ResidentBird,MigratoryBird
\end{lstlisting}
gives
\begin{lstlisting}
   Bird(Jo) 0.90296   ResidentBird(Jo) 0.667983   MigratoryBird(Jo) 0.235026
\end{lstlisting}
(When no MLN solver is specified in the command line, {\sc alchemy} is called by default.)

\begin{figure}[bp]
	\centering
		\includegraphics[width=1\textwidth]{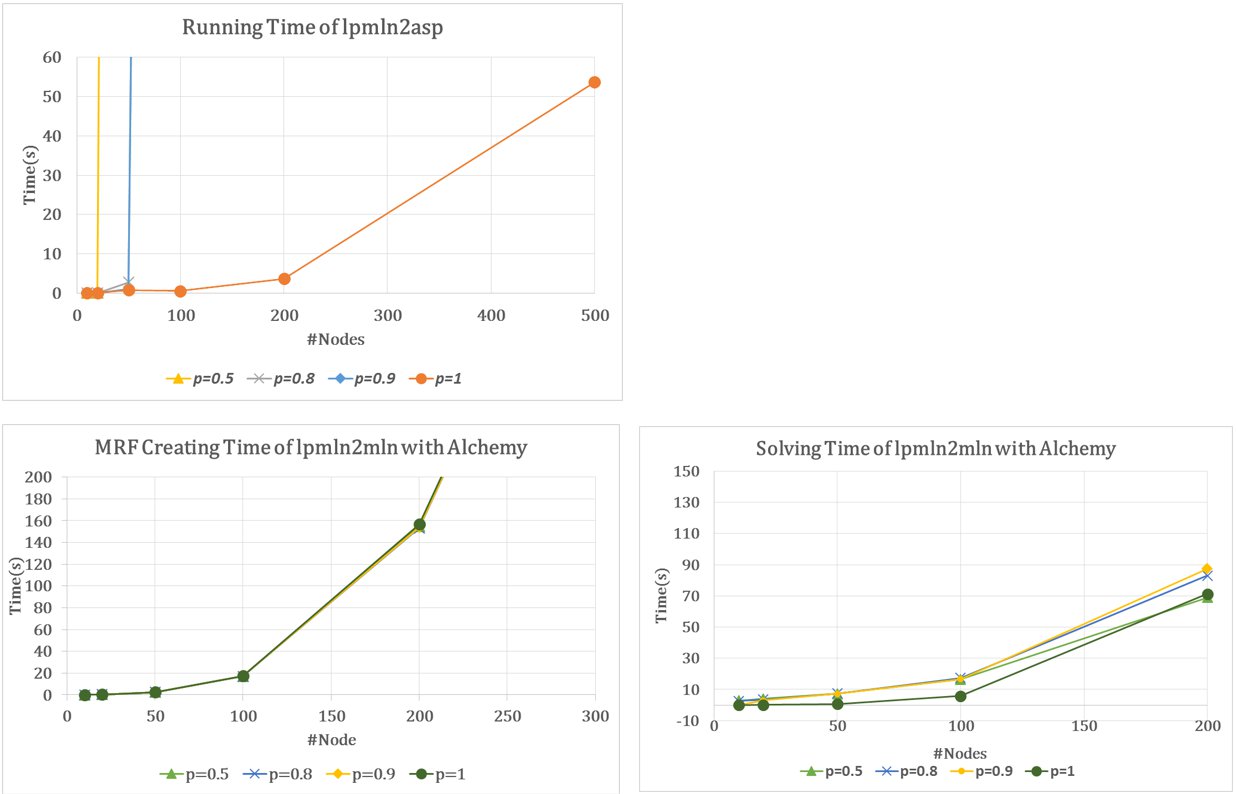}
	\caption{Running Statistics on Finding Relaxed Clique}
	\label{fig:performance}
\end{figure}
%[[lpmln learning: can't do with lpmln2mln mode]]
%\vspace{-2cm}

%\vspace{-5mm}
%--------------------------------------------------------
\section{Comparison Between Two {\sc lpmln} Implementations} \label{sec:example}
%--------------------------------------------------------

%{\cblu 
Both {\sc lpmln2asp} and {\sc lpmln2mln} can compute conditional/marginal probability, as well as finding the most probable stable models (MAP estimates). 
The implementations use ASP and MLN solvers as blackboxes, so their performance depends on the underlying solvers. 
Although ASP solvers do not have a built-in concept of probabilistic reasoning, it is interesting to note how the optimal answer set finding is related to MAP estimates in probabilistic reasoning. Grounding in ASP solvers is much more efficient than that in MLN solvers for the examples that we tested, but they have different characters. While  grounding methods implemented in MLN solvers are not highly optimized, they do not ground the whole network; rather an essential part of a Markov network can be constructed from Markov blankets relevant to the query.
Unlike {\sc lpmln2asp}, system {\sc lpmln2mln} utilizes approximate sampling-based inference methods in underlying MLN solvers. Consequently, its solving is more scalable  but gives less accurate results. Its input program is restricted to tight $\lpmln$ programs and does not support advanced ASP constructs, such as aggregates.
%}
When the domain is small, our experience is that it is much more convenient to work with {\sc lpmln2asp} because it supports many useful ASP constructs and its exact computation yields outputs that are easier to understand. Once we make sure the program is correct and we do not need advanced ASP constructs nor recursive definitions, we may use {\sc lpmln2mln} for more scalable inference.

We report the running time statistics for both {\sc lpmln2asp} and {\sc lpmln2mln} on the example of finding a maximal ``relaxed clique'' in a graph, {where the goal is to select as many nodes as possible while a penalty is assigned for each pair of disconnected nodes.  
The penalty assigned to disconnected nodes and the reward given to each node included in the subgraph define how much ``relaxed'' the clique is.}  

The {\sc lpmln2asp} encoding of the relaxed clique example is
\begin{lstlisting}
{in(X)} :- node(X).
disconnected(X, Y) :- in(X), in(Y), not edge(X, Y).
5  :- not in(X), node(X).
5  :- disconnected(X, Y).
\end{lstlisting}

The {\sc lpmln2mln} encoding of the relaxed clique example is
\begin{lstlisting}
{In(x)} <= Node(x).
Disconnected(x, y) <= In(x) ^ In(y) ^ !Edge(x, y).
5  <= !In(x) ^ Node(x)
5  <= Disconnected(x, y)
\end{lstlisting}

We use a Python script to generate random graphs with each edge generated with a fixed probability $p$. We experiment with $p=0.5,  0.8,  0.9,  1$ and different numbers of nodes. 
For each problem instance, we perform MAP inference to find a maximal relaxed clique with both {\sc lpmln2asp} and {\sc lpmln2mln}. The timeout is 20 minutes.
%We record running time if it is smaller than 20 minutes, or timeout otherwise.
The experiments are performed on a machine powered by 4 Intel(R) Core(TM) i5-2400 CPU with OS Ubuntu 14.04.5 LTS and 8G memory.

Figure \ref{fig:performance} shows running statistics of utilizing different underlying solvers.
For {\sc lpmln2asp}, grounding finishes almost instantly for all problem instances that we tested. We plot how solving times vary according to the number of nodes for different edge generation probabilities (top left graph). 
Roughly, solving time increases as the number of nodes increases. However, there is no clear correlation between solving time and the edge probability (i.e., the density of the graph). For $p=0.5$, the {\sc lpmln2asp} system first times out when $\#Nodes = 50$, while for both $p=0.8$ and $p=0.9$, it first times out when $\#Node= 100$. On the other hand, when $\#Node=20$, solving time roughly increases as the edge probability increases except for $p=0.5$. The running time is sensitive to particular problem instances, due to the exact optimization algorithm CDNL-OPT \cite{gebser11multi} used by {\sc clingo}, which only terminates when a true optimal solution is found. The non-deterministic nature of CDNL-OPT also brings randomness on the path through which an optimal solution is found, which makes the running time differ even among similar-sized problem instances, while in general, as the size of the graph increases, the search space gets larger, thus the solving time increases. 

For {\sc lpmln2mln} with {\sc alchemy} (bottom left and bottom right), grounding (MRF creating time) becomes the bottleneck. It increases much faster than solving time, and times out first when $\#Nodes=500$. Again, the running time increases as the number of nodes increases. On the other hand, unlike {\sc lpmln2asp}, {\sc alchemy} uses MaxWalkSAT for MAP inference, which allows a suboptimal solution to be returned. The approximate nature of the method allows relatively consistent running times for different problem instances, as long as parameters such as the maximum number of iterations/tries are fixed among all experiments. The running times are not also much affected by the edge probability. 

In general, {\sc lpmln2mln} can be more scalable via parameter setting, while {\sc lpmln2asp} grants better solution quality. {\sc lpmln2mln} with {\sc tuffy} shows a similar behavior as {\sc lpmln2mln} with {\sc alchemy}.

\vspace{-3.2mm}
% -------------------------------------------
\section{Using $\lpmln$ Implementations to Compute Other Languages}\label{sec:other}

%--------------------------------------------------------
\vspace{-1mm}
%--------------------------------------------------------
\subsection{Computing ProbLog} 
%--------------------------------------------------------

ProbLog \cite{deraedt07problog} can be viewed as a special case of the $\lpmln$ language \cite{lee16weighted}, in which soft rules are atomic facts only.
The precise relation between the semantics of the two languages is stated by \citeN{lee16weighted}.  System {\sc problog2} implements a native inference and learning algorithm which converts probabilistic inference problems into weighted model counting problems and then solves with knowledge compilation methods \cite{fierens13inference}.
We compare the performance of {\sc lpmln2asp} with that of {\sc problog2} on ProbLog input programs. We encode the problem of reachability in a probabilistic graph in both languages, and perform MAP inference (``given that there is a path between two nodes, what is the most likely graph?'') as well as marginal probability computation (``given two particular nodes, what is the probability that there exists a path between them?''). We use a Python script to generate edges with probabilities randomly assigned. For the probabilistic facts $p :: {\tt edge}(n_1,n_2)$ (\mbox{$0<p<1$}) in {\sc problog2}, we write $ln(p/(1-p)): {\tt edge}(n_1,n_2)$ for {\sc lpmln2asp},
% and $ln(1-p)$ as the weight of the rule {\tt :- edge}(X,Y), 
which makes the probability of the edge being true to be $p$ and being false to be $1-p$. 
%\NB{@Yi: explain state; given there are sets of paths?}

The path relation is defined in the input language of {\sc lpmln2asp} as
\begin{lstlisting}
path(X,Y) :- edge(X,Y).
path(X,Y) :- path(X,Z), path(Z, Y), Y != Z.
\end{lstlisting}
and in the input language of {\sc problog2} as 
\begin{lstlisting}
path(X,Y) :- edge(X,Y).
path(X,Y) :- path(X,Z), path(Z,Y), Y \== Z.
\end{lstlisting}

\begin{figure}[t]
	\centering
		\includegraphics[scale=0.29]{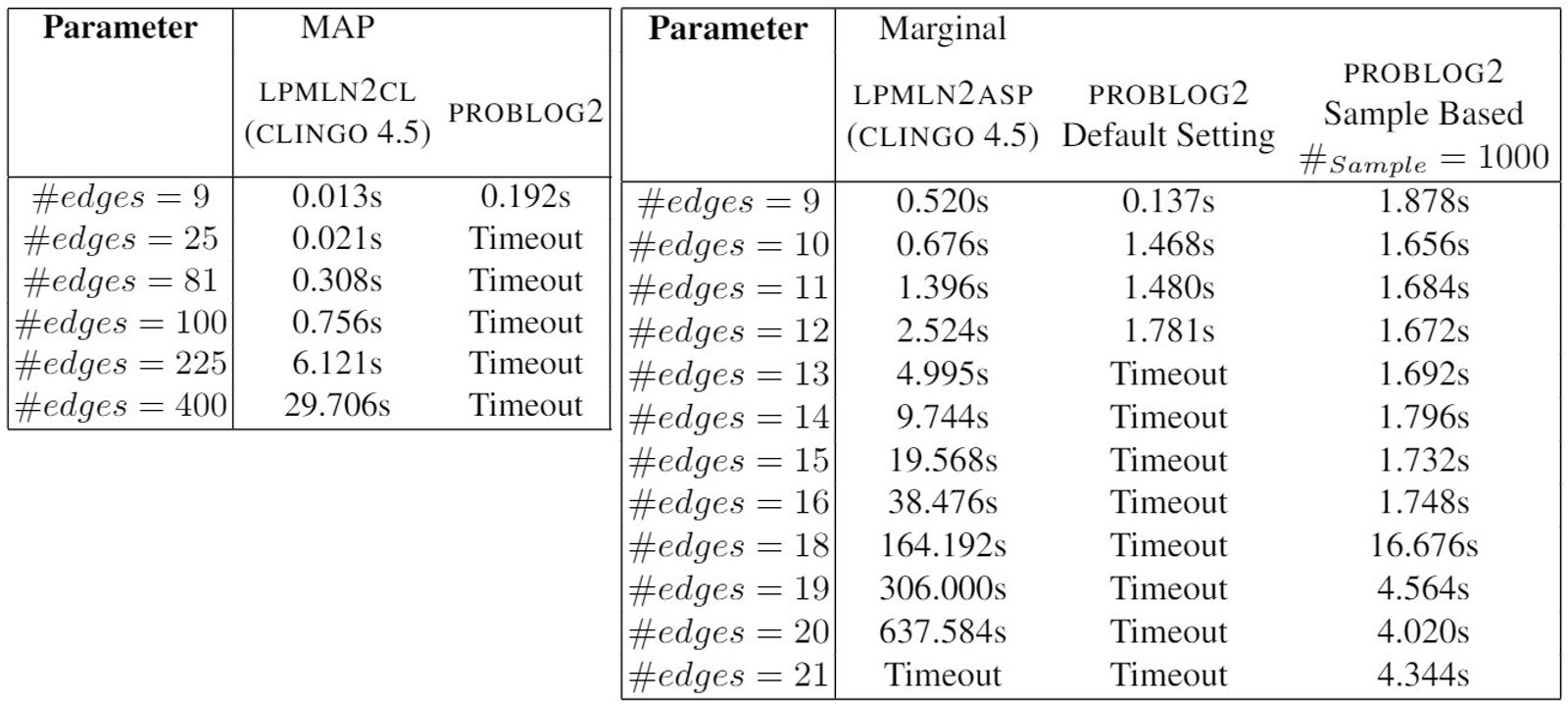}
	\caption{Running Statistics on Reachability in a Probabilistic Graph}
	\label{tab:performance-problog} \vspace{-0.5cm}
\end{figure}
Figure \ref{tab:performance-problog} shows the running time of each experiment.
%\footnote{The default mode of {\sc problog2} does exact inference} 
{\sc lpmln2asp} outperforms {\sc problog2} with the default setting (exact inference) in both MAP inference and marginal probability computation. However, both systems' marginal probability computations are not scalable because they enumerate all models. 
%{\sc lpmln2asp} and {\sc problog2} have poor performance on exact marginal inference since it requires enumerating all probabilistic stable models. On the other hand, 
Using a sampling-based inference instead, {\sc problog2} is able to handle marginal probability computation more effectively (the MAP inference in {\sc problog2} is exact inference only).
% while currently {\sc lpmln2asp} does not provide any approximate method. 
In general, compared to running on tight programs, {\sc problog2} is slow for non-tight programs such as the program we use here. A possible reason is that it has to convert the input program, combined with the query, into weighted Boolean formulas, which is expensive for non-tight programs.

\vspace{-3mm}
%--------------------------------------------------------
\subsection{Reasoning about Probabilistic Causal Model}
%--------------------------------------------------------

\begin{wrapfigure}{r}{0.38\textwidth}
%\begin{figure}[h!]
%	\centering
\vspace{-3mm}
		\includegraphics[height=4.5cm,width=5cm]{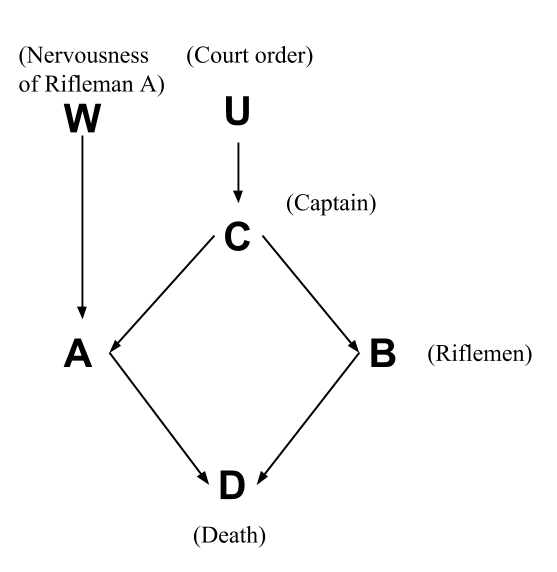}
	\caption{Firing Squad Example}
	\label{fig:FiringSquadPCM}
%\end{figure}
\end{wrapfigure}
\citeN{lee15markov} show how to represent Pearl's probabilistic causal model \cite{pearl00causality} by $\lpmln$. Due to the acyclicity assumption on the causality, the $\lpmln$ representation is tight, so we can use either implementation of $\lpmln$ to compute probabilistic queries on a PCM.
(Related to this, Appendix~A \cite{lee17computing-online} shows how Bayesian networks can be represented in $\lpmln$.)
%Probabilistic inference on a PCM can be automated by encoding the PCM as an $\lpmln$ program. 

As an example, consider a probabilistic version of the firing squad example, shown in Figure \ref{fig:FiringSquadPCM}. 
The court orders the execution ($U$) with probability $p$ and Rifleman A is nervous ($W$) with probability $q$. The nervousness of Rifleman A causes him shooting at the prisoner ($A$). The court order causes the Captain to signal ($C$), which again causes Rifleman A and Rifleman B to shoot at the prisoner. Either of Rifleman A and Rifleman B shooting causes the prisoner's death ($D$). We illustrate how we use $\lpmln$ systems to compute the counterfactual query ``Given that the prisoner is dead, what is the probability that the prisoner would be alive if Rifleman A had not shot?'' According to \citeN{pearl00causality}, the answer is $\frac{(1-p)q}{1-(1-p)(1-q)}$.

Theorem~4 from the paper by~\citeN{lee15markov} states that the counterfactual reasoning in PCM can be reduced to $\lpmln$ computation. The translation of PCM into $\lpmln$ in Section 4.4 by~\citeN{lee15markov} can be represented in the input language of {\sc lpmln2asp} as follows, where {\tt as}, {\tt bs}, {\tt cs}, {\tt ds} are nodes in the twin network, {\tt a1} means that {\tt a} is true; {\tt a0} means that {\tt a} is false; other atoms are defined similarly. Let $p=0.7$ and $q=0.2$.

\begin{multicols}{2}
\begin{lstlisting}
@log(0.7/0.3)  u.
@log(0.2/0.8)  w.

c :- u. 
a :- c.
a :- w.
b :- c.
d :- a.
d :- b.
\end{lstlisting}
\end{multicols}
\vspace{-0.7cm}
\begin{multicols}{2}
\begin{lstlisting}
cs :- u, not do(c1), not do(c0). 
as :- cs, not do(a1), not do(a0).
as :- w, not do(a1), not do(a0).
bs :- cs, not do(b1), not do(b0).
ds :- as, not do(d1), not do(d0).
ds :- bs, not do(d1), not do(d0).

cs :- do(c1).
as :- do(a1).
bs :- do(b1).
ds :- do(d1).
\end{lstlisting}
\end{multicols}

To represent the counterfactual query, the evidence file contains: 
\begin{lstlisting}
do(a0).
:- not d. 
\end{lstlisting}
Note the different ways that intervention ({\tt do(a0)}) and observation ({\tt d}) are encoded. 

With the command
%\begin{lstlisting}
{\tt lpmln2asp -i pcm.lp -r out -e evid.db -q ds}
%\end{lstlisting}
we obtain 
%\begin{lstlisting}
{\tt ds 0.921047297896},
%\end{lstlisting}
which means there is a $8\%$ chance that the prisoner would be alive. 

%\vspace{-5mm}
%--------------------------------------------------------
\section{Conclusion}\label{sec:conclusion}
%--------------------------------------------------------

We presented two implementations of $\lpmln$ using ASP and MLN solvers. 
This is based on extending the translations that turn $\lpmln$ into answer set programs and Markov logic to allow non-ground weighted rules.
%, as well as relating the notion of penalty in weak constraints to the notion of probability in $\lpmln$.
%
Although the input language of {\sc clingo} does not have a built-in concept of probabilistic reasoning, its optimal answer set finding algorithm is shown to be effective in finding MAP estimates (most probable stable models). It is also interesting that the efficient stable model enumeration leads to competitive exact probability computation.

The implementations also serve for other probabilistic logic languages that are shown to be embeddable in $\lpmln$, such as ProbLog, Pearl's causal models, Bayesian networks, and P-log.

PrASP \cite{nickles16atool} is another system whose input language extends answer set programs with weights, although the semantics is different from that of $\lpmln$.
While $\lpmln$ systems turn an input program into another input program that can be computed by existing systems,  PrASP implements several native inference algorithms, including model counting, simulated annealing, flip-sampling, iterative refinement, etc.
The formal relationships between the language of PrASP and other languages have not been established.

The $\lpmln$ implementations suggest how to combine the solving techniques from the two solvers. While {\sc clingo} is efficient for grounding, MLN solvers consider subnetworks derived from the Markov blanket of query atoms and evidence. While {\sc clingo} does exact inference only, MLN solvers can perform sampling based approximate inference. Future work includes building a native algorithm for $\lpmln$ borrowing the techniques from the related systems. 

\medskip\noindent
{\bf Acknowledgements:} 
We are grateful to Zhun Yang, Brandon Gardell and the anonymous referees for their useful comments. This work was partially supported by the National Science Foundation under Grants IIS-1319794 and IIS-1526301.

%\vspace{-3mm}
\bibliographystyle{acmtrans}
%\bibliography{bib,bib2}

\include{lpmln-system-online-appendix-1121}

\end{document}

%% file: lpmln-system-online-appendix-1121.tex
 %\setcounter{page}{1}
 \BOC
\title{Appendix: Computing $\lpmln$ Using ASP and MLN Solvers} 

\begin{center}
{\large\textnormal{Online appendix for the paper}}   \\
\medskip
{\Large {\sl Computing $\lpmln$ Using ASP and MLN Solvers}
\\
\medskip
{\large\textnormal{published in Theory and Practice of Logic Programming}}
}

\medskip
Joohyung Lee, Samidh Talsania, and Yi Wang \\ 
{\sl School of Computing, Informatics and Decision Systems Engineering \\
Arizona State University, Tempe, AZ, USA}

%\author[Lee, Loney \& Meng]{Joohyung Lee, Nikhil Loney, and Yunsong Meng}
\end{center}

%\label{firstpage}
\thispagestyle{empty}
\EOC

\begin{appendix}

\section{Bayesian Network in $\lpmln$} \label{ssec:bayes-net}

It is easy to represent Bayesian networks in $\lpmln$
similar to the way Bayesian networks are represented by weighted Boolean formulas \cite{sang05solving}.

We assume all random variables are Boolean.
Each conditional probability table associated with the nodes can be represented by a set of probabilistic facts. For each CPT entry $P(V=\true \mid V_1={S_1}, \dots, V_n={S_n}) = p$ where 
$S_1,\dots, S_n \in\{\true, \false\}$, we include a set of weighted facts
\bi
\ii $ln(p/(1-p)):\ \  PF(V,S_1,\dots,S_n)$ if $0<p<1$;
\ii $\alpha:\ \ \  PF(V, S_1,\dots,S_n)$ if $p=1$;
\ii $\alpha:\ \ \  \ar\ \no\ PF(V, S_1,\dots,S_n)$ if $p=0$.
\ei
For each node $V$ whose parents are $V_1, \dots, V_n$, the directed edges can be represented by rules 
\[   
   \alpha:\ V \ar V_1^{S_1}, \dots, V_n^{S_n}, PF(V, S_1, \dots, S_n)   \ \qquad \ (S_1,\dots, S_n \in\{\true, \false\})
\]
where $V_i^{S_i}$ is $V_i$ if $S_i$ is \true, and $\no\ V_i$ otherwise. 

For example, in the firing example in Figure~\ref{fig:bayes-example}, the conditional probability table for the node ``alarm" can be represented by 
%\begin{multicols}{2}
%\begin{lstlisting}[mathescape=true]
%$ln(0.5/0.5)$           pf(a,t1f1). 
%$ln(0.85/0.15)$      pf(a,t1f0). 
%$ln(0.99/0.01)$               pf(a,t0f1). 
%$ln(0.0001/0.0009)$      pf(a,t0f0). 
%\end{lstlisting}
%\end{multicols}

\begin{multicols}{2}
\begin{lstlisting}
@log(0.5/0.5)     pf(a,t1f1). 
@log(0.85/0.15)   pf(a,t1f0). 
@log(0.99/0.01)      pf(a,t0f1). 
@log(0.0001/0.0009)  pf(a,t0f0). 
\end{lstlisting}
\end{multicols}

The directed edges can be represented by hard rules as follows: 
\begin{multicols}{2}
%-3.8918 pf(t). \\
%-4.5951 pf(f). \\
%0  pf(a,t1f1). \\
%1.7346  pf(a,t1f0). \\
%4.5951  pf(a,t0f1). \\
%-9.2102 pf(a,t0f0). \\
%
%2.1972  pf(s,f1). \\
%-4.5951 pf(s,f0).  \\
%
%1.9924   pf(l,a1). \\
%-6.9068 pf(l,a0). \\
%
%1.0986 pf(r,l1). \\ 
%-4.5951 pf(r,l0). \\
%

\lstset{
   basicstyle=\small\ttfamily,
   basewidth=0.5em,
   numbers=none,
   numberstyle=\tiny,  
   stringstyle=\small\ttfamily,
   showspaces=false,
   showstringspaces=false
}

\begin{lstlisting}
tampering :- pf(t). 

fire :- pf(f). 

alarm :- tampering, fire, pf(a,t1f1).   
alarm :- tampering, not fire, pf(a,t1f0). 
alarm :- not tampering, fire, pf(a,t0f1). 
alarm :- not tampering, not fire, pf(a,t0f0).  

    smoke :- fire, pf(s,f1).  
    smoke :- not fire, pf(s,f0). 

    leaving :- alarm, pf(l,a1). 
    leaving :- not alarm, pf(l,a0). 

    report :- leaving, pf(r,l1). 
    report :- not leaving, pf(r,l0). 
\end{lstlisting}
\end{multicols}

\begin{figure}[b]
		\includegraphics[height=5cm]{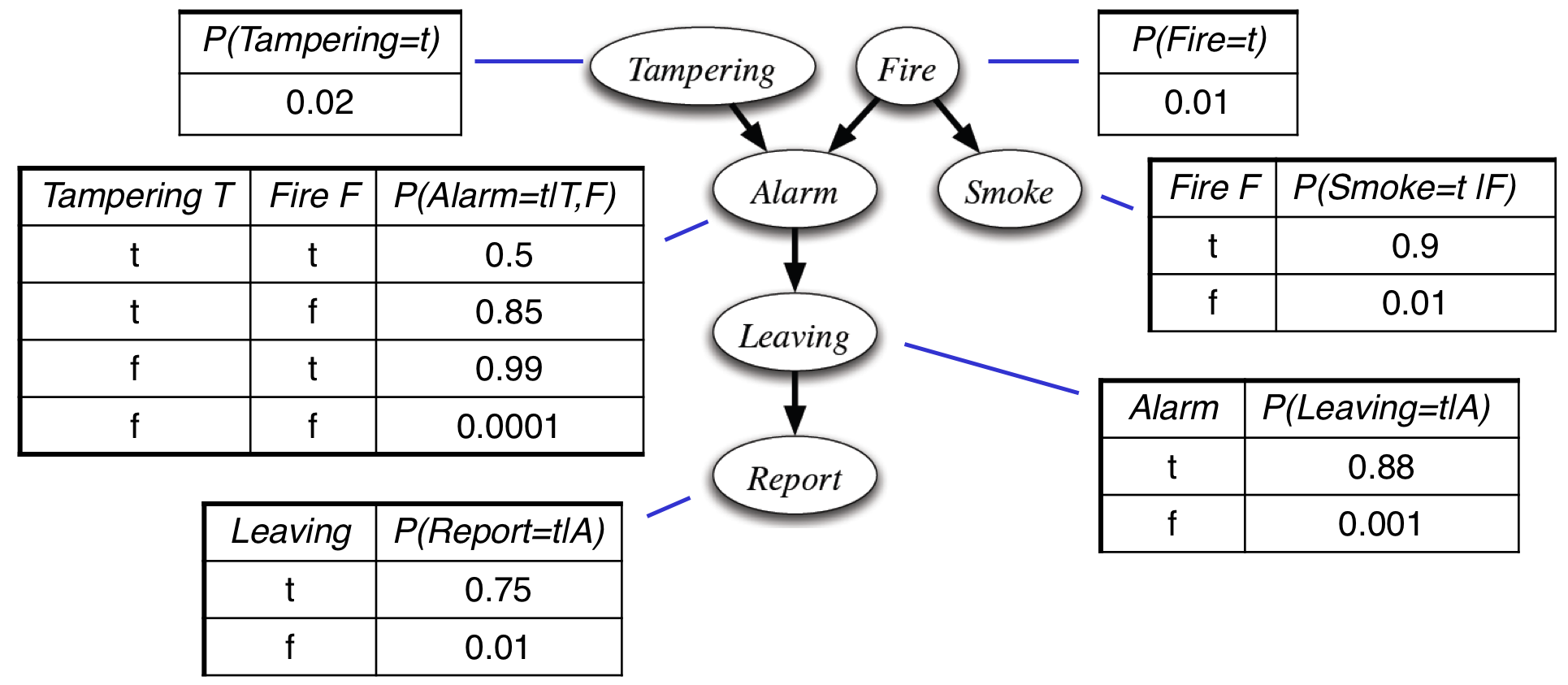}
	\caption{Bayes Net Example} 
	\label{fig:bayes-example}
\end{figure}

\begin{thm}
For any Bayesian network whose random variables are Boolean and any interpretation $I$, the probability of $I$ according to the Bayesian network semantics coincides with the probability of $I$ for the translated $\lpmln$ program. 
\end{thm}

Since Bayesian networks are represented by directed acyclic graphs, $\lpmln$ programs that represent them are always tight. So both {\sc lpmln2asp} and {\sc lpmln2mln} can be used to compute Bayesian networks. 

\begin{itemize}
%-----------------------------
\item \emph{Diagnostic Inference} is to compute the probability of the cause given the effect. For example, to compute $P(fire =\mathbf{t} \mid leaving=\mathbf{t})$, the user can invoke 
\begin{lstlisting}
lpmln2asp -i fire-bayes.lpmln -e evid.db -q fire
\end{lstlisting}
where {\tt evid.db} contains the line 
\begin{lstlisting}
:- not leaving. 
\end{lstlisting}
This outputs
\begin{lstlisting}
fire 0.352151116689
\end{lstlisting}

%-----------------------------
\item \emph{Predictive Inference} is to compute the probability of the effect given the cause. For example, to compute $P(leaving=\mathbf{t} \mid  fire=\mathbf{t})$, the user can invoke 
\begin{lstlisting}
lpmln2asp -i fire-bayes.lpmln -e evid.db  -q leaving
\end{lstlisting}
where {\tt evid.db} contains the line 
\begin{lstlisting}
:- not fire. 
\end{lstlisting}
This outputs
\begin{lstlisting}
leaving 0.862603541626
\end{lstlisting}

%-----------------------------
\item \emph{Mixed Inference} is to combine \emph{predictive} and \emph{diagnostic} inference. For example, to compute $P(alarm=\mathbf{t} \mid fire=\mathbf{f}, leaving=\mathbf{t})$, the user can invoke 
\begin{lstlisting}
lpmln2asp -i fire-bayes.lpmln -e evid.db  -q alarm
\end{lstlisting}
where {\tt evid.db} contains two lines {\tt } 
\begin{lstlisting}
:- fire.
:- not leaving.
\end{lstlisting}
This outputs
\begin{lstlisting}
alarm 0.938679679707
\end{lstlisting}

%-----------------------------
\item \emph{Intercausal Inference} is to compute the probability of a cause given an effect common to multiple causes. For example, to compute $P(tampering=\mathbf{t} \mid fire=\mathbf{t}, alarm=\mathbf{t})$, the user can invoke 
\begin{lstlisting}
lpmln2asp -i fire-bayes.lpmln -e evid.db  -q tampering
\end{lstlisting}
where {\tt evid.db} contains two lines 
\begin{lstlisting}
:- not fire.
:- not alarm.
\end{lstlisting}
This outputs
\begin{lstlisting}
tampering 0.0102021964693
\end{lstlisting}

%-----------------------------
\item \emph{Explaining away}: Suppose we know that $alarm$ rang. Then we can use \emph{Diagnostic Inference} to calculate $P(tampering=\mathbf{t} \mid alarm=\mathbf{t})$. But what happens if we now know that there was a $fire$ as well? In this case $P(tampering=\mathbf{t} \mid alarm=\mathbf{t})$ will change to $P(tampering=\mathbf{t} \mid fire=\mathbf{t}, alarm=\mathbf{t})$. In this case, knowing that there was a $fire$ explains away $alarm$, and hence affecting the probability of $tampering$. 
%Even though $fire$ and $tampering$ are independent, the knowledge about one changes the probability of other.
For example, to compute $P(tampering=\mathbf{t} \mid alarm=\mathbf{t})$, 
%which states the probability of $tampering$ to be true given $alarm$ is true. 
the user can invoke 
\begin{lstlisting}
lpmln2asp -i fire-bayes.lpmln -e evid.db  -q tampering
\end{lstlisting}
where {\tt evid.db} contains line 
\begin{lstlisting}
:- not alarm.
\end{lstlisting}
This outputs
\begin{lstlisting}
tampering 0.633397289908
\end{lstlisting}
If we compare this result with the result of \emph{Intercausal Inference}, we see that $P(tampering=\mathbf{t} \mid alarm=\mathbf{t}) > P(tampering=\mathbf{t} \mid fire=\mathbf{t}, alarm=\mathbf{t})$. Observing the value of $fire$ explains away the $tampering$ i.e., the probability of $tampering$ decreases. 

\end{itemize}
\BOCC
\bi
\ii To compute $P(fire \mid alarm)$, one can invoke 
\begin{lstlisting}
   lpmln2mln -i fire-bayes.lpmln -e evid1.db -r output -q fire
\end{lstlisting}
where {\tt evid1.db} contains the line {\tt alarm}. 

\ii To compute $P(fire \mid  alarm, \neg tampering)$, one can invoke 
\begin{lstlisting}
   lpmln2mln -i fire-bayes.lpmln -e evid2.db -r output -q fire
\end{lstlisting}
where {\tt evid1.db} contains two lines {\tt alarm} and {\tt !tampering}. 
\ei
\EOCC

%------------------------------------------------------------------
\section{Proof of Theorem \ref{thm:lpmln-pnt}} 
%------------------------------------------------------------------

\noindent{\bf Theorem~\ref{thm:lpmln-pnt} \optional{thm:lpmln-pnt}}\
\ 
For any $\lpmln$ program $\Pi$ and any interpretation $I$, 
\[
  W_{\Pi}(I) \propto W^{\rm pnt}_{\Pi}(I)% \times TW_{\Pi}
\text{ \ \ \ \ and \ \ \ \ }
  P_\Pi(I) = P_\Pi^{\rm pnt}(I).
\]

\begin{proof}
Let 
\[
TW_{\Pi} = exp\left(\sum_{w: F \in \Pi}w\right).
\]

We first show that $W_{\Pi}(I) = TW_{\Pi}\cdot W^{\rm pnt}_{\Pi}(I)$. This is obvious when $I\notin \sm[\Pi]$. 

When $I\in \sm[\Pi]$, we have
\begin{align*}
  W_\Pi(I) &= exp\bigg(\sum_{\text{$w:F\in\Pi$ and $I\models  F$}}w\bigg) \\
               &= exp\bigg(\sum_{w: F \in \Pi} w - \sum_{\text{$w:F\in\Pi$ and $I\not\models F$}} w)\\
               &= exp\bigg(\sum_{w: F \in \Pi}w\bigg)\cdot exp\bigg(-\sum_{\text{$w:F\in\Pi$ and $I\not\models F$}}w\bigg)\\
               &= TW_{\Pi}\cdot exp\bigg(- \sum_{\text{$w:F\in\Pi$ and $I\not\models  F$}}w\bigg)\\
&= TW_{\Pi}\cdot W^{\rm pnt}_{\Pi}(I).
\end{align*}
Consequently,
\begin{align*}
  P_\Pi(I) &= \frac{W_{\Pi}(I)}{\sum_J W_{\Pi}(J)}\\
              &= \frac{TW_{\Pi}\cdot W^{\rm pnt}_{\Pi}(I)}{\sum_J TW_{\Pi}\cdot W^{\rm pnt}_{\Pi}(J)}\\
              &= \frac{W^{\rm pnt}_{\Pi}(I)}{\sum_J W^{\rm pnt}_{\Pi}(J)}\cdot \frac{TW_{\Pi}}{TW_{\Pi}}\\
              &= \frac{W^{\rm pnt}_{\Pi}(I)}{\sum_J W^{\rm pnt}_{\Pi}(J)}\\
              &= P_\Pi^{\rm pnt}(I).
\end{align*}
\end{proof}

\BOCC
%------------------------------------------------------------------
\section{Proof of Theorem \ref{thm:lpmln-pnt}} 
 %------------------------------------------------------------------

\noindent{\bf Theorem~\ref{thm:lpmln-pnt} \optional{thm:lpmln-pnt}}\\
 
For any $\lpmln$ program $\Pi$ and any interpretation $I$, 
\[
  W_{\Pi}(I) \propto W^{\rm pnt}_{\Pi}(I)% \times TW_{\Pi}
\text{ \ \ \ \ and \ \ \ \ }
  P_\Pi(I) = P_\Pi^{\rm pnt}(I).
\]

%\begin{proof}
\proof
Let 
\[  
  TW_\Pi = exp\left(\sum_{w: F \in \Pi}w\right).
\]
be the "total" weight of $\Pi$.

\begin{align}
\nonumber W_{\Pi}(I) &= exp(\sum_{\text{$w:F\in\Pi$ and $I\models  F$}}w) \\
\nonumber &= exp(\sum_{w: F \in \Pi}w - \sum_{\text{$w:F\in\Pi$ and $I\not\models  F$}}w)\\
\nonumber &= exp(\sum_{w: F \in \Pi}w)\cdot exp(- \sum_{\text{$w:F\in\Pi$ and $I\not\models  F$}}w)\\
\nonumber &= TW_{\Pi}\cdot exp(- \sum_{\text{$w:F\in\Pi$ and $I\not\models  F$}}w)\\
\nonumber &= TW_{\Pi}\cdot W^{\rm pnt}_{\Pi}(I)
\end{align}
Consequently,
\begin{align}
\nonumber P_{\Pi}(I) &= \frac{W_{\Pi}(I)}{\sum_J W_{\Pi}(J)}\\
\nonumber &= \frac{TW_{\Pi}\cdot W^{\rm pnt}_{\Pi}(I)}{\sum_J TW_{\Pi}\cdot W^{\rm pnt}_{\Pi}(J)}\\
\nonumber &= \frac{W^{\rm pnt}_{\Pi}(I)}{\sum_J W^{\rm pnt}_{\Pi}(J)}\cdot \frac{\sum_J TW_{\Pi}}{\sum_J TW_{\Pi}}\\
\nonumber &= \frac{W^{\rm pnt}_{\Pi}(I)}{\sum_J W^{\rm pnt}_{\Pi}(J)}\\
\nonumber &= P_\Pi^{\rm pnt}(I)
\end{align}
\qed
%\end{proof}
\EOCC

%------------------------------------------------------------------
\section{Proof of Theorem \ref{thm:lpmln2asp-rwd}} 
 %------------------------------------------------------------------

%[[Theorem and proof]]
We divide the ground program obtained from ${\sf lpmln2asp^{rwd}}(\Pi)$ into three parts:
\[
    SAT(\Pi) \cup ORIGIN(\Pi) \cup WC(\Pi)
\]
where
\begin{align*}
  SAT(\Pi) =& \{{\tt sat}(i, w_i, {\bf c}) \ar  \j{Head}_i({\bf c}) 
      \mid w_i:\ \j{Head}_i({\bf c}) \ar\j{Body}_i({\bf c})\in Gr(\Pi)\}\ \cup\\
                  & \{{\tt sat}(i, w_i, {\bf c}) \ar  {\tt not}\ \j{Body}_i({\bf c}) 
      \mid w_i:\ \j{Head}_i({\bf c}) \ar \j{Body}_i({\bf c})\in Gr(\Pi)\}
\end{align*}
\[
   ORIGIN(\Pi) = \{\j{Head}_i({\bf c})\ar \j{Body}_i({\bf c}), {\tt not}\ {\tt not}\ {\tt sat}(i, w_i, {\bf c}) 
      \mid w_i:\ \j{Head}_i({\bf c}) \ar \j{Body}_i({\bf c})\in Gr(\Pi)\}
\]
and 
\[
   WC(\Pi) =\{:\sim {\tt sat}(i, w_i, {\bf c}). \ \ [-w_i@l, i, {\bf c}]
       \mid w_i:\  \j{Head}_i({\bf c}) \leftarrow \j{Body}_i({\bf c})\in Gr(\Pi)\}
\]
\begin{lemma}\label{lem:lpmln2asp_rwd}
For any $\lpmln$ program $\Pi$,
\[
    \phi(I)=I\cup \{{\tt sat}(i, w_i, {\bf c}) \mid w_i:\j{Head}_i({\bf c}) \ar
     \j{Body}_i({\bf c})\in Gr(\Pi), I\models  \j{Head}_i({\bf c}) \leftarrow \j{Body}_i({\bf c})\}
\] 
is a 1-1 correspondence between $\sm[\Pi]$ and the stable models of $SAT(\Pi)\cup ORIGIN(\Pi)$.
\end{lemma}

\begin{proof}
Let $\sigma$ be the signature of $\Pi$, and let $\sigma_{sat}$ be the set 
\[
  \{{\tt sat}(i, w_i, {\bf c}) \mid  
         w_i:\ \j{Head}_i({\bf c}) \ar \j{Body}_i({\bf c})\in Gr(\Pi)\}.
\]
It can be seen that
\begin{itemize}
\item each strongly connected component of the dependency graph of $SAT(\Pi)\cup ORIGIN(\Pi)$ w.r.t. $\sigma\cup \sigma_{sat}$ is a subset of $\sigma$ or a subset of $\sigma_{sat}$;

\item no atom in $\sigma_{sat}$ has a strictly positive occurrence in $ORIGIN(\Pi)$;

\item no atom in $\sigma$ has a strictly positive occurrence in $SAT(\Pi)$.
\end{itemize}
Thus, according to the splitting theorem, $\phi(I)$ is a stable model of $SAT(\Pi)\cup ORIGIN(\Pi)$ if and only if $\phi(I)$ is a stable model of $SAT(\Pi)$ w.r.t. $\sigma_{sat}$ and is a stable model of $ORIGIN(\Pi)$ w.r.t. $\sigma$.

First, assuming that $I$ belongs to $\sm[\Pi]$, we will prove that $\phi(I)$ is a stable model of $SAT(\Pi)\cup ORIGIN(\Pi)$. Let $I$ be a member of $\sm[\Pi]$. 
\begin{itemize}
\item {\bf $\phi(I)$ is a stable model of $SAT(\Pi)$ w.r.t. $\sigma_{sat}$.}\ \  \ \ 
By the definition of $\phi$, ${\tt sat}(i, w_i, {\bf c})\in \phi(I)$ if and only if $I\models \j{Head}_i({\bf c})\ar \j{Body}_i({\bf c})$, in which case either $I\models \j{Head}_i({\bf c})$ or $I\not\models \j{Body}_i({\bf c})$. 
This means 
\[ 
    \phi(I)\models  SAT(\Pi) \cup \{{\tt sat}(i, w_i, {\bf c})\rightarrow \j{Head}_i({\bf c}) \vee \neg\j{Body}_i({\bf c})\mid w_i:\j{Head}_i({\bf c}) \ar\ \j{Body}_i({\bf c})\in Gr(\Pi)\},
\] 
which is the completion of $SAT(\Pi)$. 
It is obvious that $SAT(\Pi)$ is tight on $\sigma_{sat}$. So $\phi(I)$ is a stable model of $SAT(\Pi)$ w.r.t. $\sigma_{sat}$.

\item {\bf $\phi(I)$ is a stable model of $ORIGIN(\Pi)$ w.r.t. $\sigma$.}\ \ \ \ 
It is clear that $\phi(I)$ satisfies $ORIGIN(\Pi)$. Assume for the sake of contradiction that there is an interpretation $J\subset \phi(I)$ such that $J$ and $\phi(I)$ agree on $\sigma^{sat}$ and $J\models ORIGIN(\Pi)^{\phi(I)}$. Then 
\[ 
   J\models \j{Head}_i({\bf c})^{\phi(I)}\ar\j{Body}_i({\bf c})^{\phi(I)}, ({\tt not}\ {\tt not}\ {\tt sat}(i,w_i,{\bf c}))^{\phi(I)}
\]
for every rule
\[ 
   \j{Head}_i({\bf c})\ar\j{Body}_i({\bf c}), {\tt not}\ {\tt not}\ {\tt sat}(i, w_i, {\bf c})
\] 
in $ORIGIN(\Pi)$.
Since $\phi(I)$ satisfies $SAT(\Pi)$,  it follows that for every rule $\j{Head}_i({\bf c})\ar\j{Body}_i({\bf c})$ satisfied by $\phi(I)$, we have 
$({\tt not}\ {\tt not}\ {\tt sat}(i,w_i,{\bf c}))^{\phi(I)}=\top$ so that 
$J\models \j{Head}_i({\bf c})^{\phi(I)}\ar\j{Body}_i({\bf c})^{\phi(I)}$, 
or equivalently, 
$J\models \j{Head}_i({\bf c})^I \ar\j{Body}_i({\bf c})^I$, 
which contradicts that $I$ is a stable model of $\overline{\Pi_I}$.
\end{itemize} 

Consequently, by the splitting theorem, $\phi(I)$ is a stable model of $SAT(\Pi)\cup ORIGIN(\Pi)$.

\bigskip
Next, assuming $\phi(I)$ is a stable model of $SAT(\Pi)\cup ORIGIN(\Pi)$, we will prove that $I$ belongs to $\sm[\Pi]$.

Let $\phi(I)$ be a stable model of $SAT(\Pi)\cup ORIGIN(\Pi)$. By the splitting theorem, $\phi(I)$ is a stable model of $SAT(\Pi)$ w.r.t. $\sigma_{sat}$ and $\phi(I)$ is a stable model of $ORIGIN(\Pi)$ w.r.t. $\sigma$. 
%So $\phi(I)$ satisfies $ORIGIN(\Pi)$. 

It is clear that $I\models \overline{\Pi_I}$. 

Assume for the sake of contradiction that there is an interpretation 
$J\subset I$ such that $J\models (\overline{\Pi_I})^I$. 
Take any rule 
\beq
   (\j{Head}_i({\bf c}))^{\phi(I)} \ar (\j{Body}_i({\bf c}))^{\phi(I)}, ({\tt not}\ {\tt not}\ {\tt sat}(i, w_i, {\bf c}))^{\phi(I)}
\eeq{origin-reduct}
in $(ORIGIN(\Pi))^{\phi(I)}$.

\medskip\noindent
{\sl Case 1}: $\phi(I)\not\models {\tt sat}(i, w_i, {\bf c})$. Clearly, $J\models \eqref{origin-reduct}$.

\medskip\noindent
{\sl Case 2}: $\phi(I)\models {\tt sat}(i, w_i, {\bf c})$. Since $\j{Head}_i({\bf c})$ and $\j{Body}_i({\bf c})$ do not contain ${\tt sat}$ predicates, \eqref{origin-reduct} is equivalent to 
\beq
  (\j{Head}_i({\bf c}))^I\ar(\j{Body}_i({\bf c}))^I.
\eeq{origin-reduct2}
Since $\phi(I)$ is a stable model of $SAT(\Pi)$ w.r.t. $\sigma_{sat}$, we have $\phi(I)\models \j{Head}_i({\bf c})\ar \j{Body}_i({\bf c})$, or equivalently, 
$I\models \j{Head}_i({\bf c})\ar \j{Body}_i({\bf c})$. So, $\j{Head}_i({\bf c})\ar\j{Body}_i({\bf c}) \in \overline{\Pi_I}$, and 
$\j{Head}_i({\bf c})^I\ar\j{Body}_i({\bf c})^I \in (\overline{\Pi_I})^I$.
Since $J\models (\overline{\Pi_I})^I$, it follows that  $J\models \eqref{origin-reduct}$ as well.

Since $J\subset \phi(I)$, $\phi(I)$ is not a stable model of $ORIGIN(\Pi)$ w.r.t. $\sigma$, which contradicts the assumption that it is.
Thus we conclude that $I$ is a stable model of~$\overline{\Pi_I}$, i.e., $I$ belongs to $\sm[\Pi]$.
\end{proof}

%====
%\newpage
%
%which is satisfied by $J$, so that $J\models \eqref{origin-reduct}$.
%
%From $J\subset I$, it follows that $J\subset \phi(I)$, which contradict that $\phi(I)$ is a stable model of $ORIGIN(\Pi)$.
%Consequently, $I$ is a stable model of~$\overline{\Pi}_I$.
%
%===
%
%For every rule $w_i: \j{Head}_i({\bf c})\ar\j{Body}_i({\bf c})$ in $\overline{\Pi_I}$, $\phi(I)\models {\tt sat}(i,w_i,{\bf c})$ so that 
%\beq
%   J\models \j{Head}_i({\bf c})^{\phi(I)}\ar \j{Body}_i({\bf c})^{\phi(I)},
%      ({\tt not}\ {\tt not}\ {\tt sat}(i,w_i,{\bf c}))^{\phi(I)},
%\eeq{origin-reduct2}
%for every rule 
%\[ 
%   \j{Head}_i({\bf c})\ar\j{Body}_i({\bf c}), {\tt not}\ {\tt not}\ {\tt sat}(i, w_i, {\bf c})
%\] 
%in $ORIGIN(\Pi)$ such that $\phi(I)\models {\tt sat}(i,w_i,{\bf c})$. For other rules in $ORIGIN(\Pi)$ such that $\phi(I)\not\models {\tt sat}(i,w_i,{\bf c})$, clearly, \eqref{origin-reduct} is true. Consequently, $J\models ORIGIN(\Pi)^{\phi(I)}$. From $J\subset I$, it follows that $J\subset \phi(I)$, which contradict that $\phi(I)$ is a stable model of $ORIGIN(\Pi)$.
%Consequently, $I$ is a stable model of~$\overline{\Pi}_I$.
%\end{proof}

\BOCC

Since $I$ is a model of $\overline{\Pi_I}$, and ${\tt sat}(i, w_i, {\bf c})\in \phi(I)$ if and only if $I\models {\i Head}_i({\bf c})\ar {\i Body_i}({\bf c})$, $\phi(I)$ is a model of $ORIGIN(\Pi)$. 
Next, we show that for any interpretations $J$ and $I$ and any rule 
$\j{Head}_i({\bf c})^I\ar \j{Body}_i({\bf c})$ in $\Pi_I$ that is satisfied by $I$, we have
$J\models \j{Head}_i({\bf c})^I\ar \j{Body}_i({\bf c})^I$ 
iff
$J\models \j{Head}_i({\bf c})^I\ar \j{Body}_i({\bf c}), {\tt not}\ {\tt not}\ {\tt sat}(i, w_i, {\bf c})^I$ 

-====

Next we show that $\phi(I)$ satisfies the loop formula of $ORIGIN(\Pi)$. Let $L$ be any subset of $\sigma$ that $\phi(I)$ satisfies. Since $I$ is a stable model of $\overline{\Pi_I}$, we have
\[
I \models  LF_{\overline{\Pi_I}}(L)
\]
i.e., 
\[
I \models  L^{\wedge} \rightarrow \bigvee_{\substack{Head_i({\bf c})\cap L\neq \emptyset, \\Head_i({\bf c}) \leftarrow Body_i({\bf c})\in\overline{\Pi_I},\\ Body_i({\bf c})\cap L =\emptyset}} (Body_i({\bf c})\bigwedge_{b\in Head_i({\bf c})\setminus L} \neg b)
\]
Since $I\models  L$, $I\models  Head_i({\bf c})$ for all $Head_i({\bf c})\cap L\neq \emptyset$, and thus $\phi(I) \models  {\tt sat}(i, w_i, {\bf c})$, so we have
\[
\phi(I) \models  L^{\wedge} \rightarrow \bigvee_{\substack{Head_i({\bf c})\cap L\neq \emptyset,\\ Head_i({\bf c}) \leftarrow Body_i({\bf c})\in\overline{\Pi_I}, \\Body_i({\bf c})\cap L =\emptyset}} (Body_i({\bf c})\wedge \neg\neg {\tt sat}(i, w_i, {\bf c})\bigwedge_{b\in Head_i({\bf c})\setminus L} \neg b)
\]
Since $\{Head_i({\bf c}) \leftarrow Body_i({\bf c}), not\ not\ {\tt sat}(i, w_i, {\bf c})\mid Head_i({\bf c}) \leftarrow Body_i({\bf c}) \in \overline{\Pi_i}\}$ is a subset of $ORIGIN(\Pi)$, we have
\[
\phi(I) \models  L^{\wedge} \rightarrow \bigvee_{\substack{Head_i({\bf c})\cap L\neq \emptyset, \\Head_i({\bf c}) \leftarrow Body_i({\bf c}), \\not\ not\ {\tt sat}(i, w_i, {\bf c})\in ORIGIN(\Pi),\\ Body_i({\bf c})\cap L =\emptyset}} (Body_i({\bf c})\wedge \neg\neg {\tt sat}(i, w_i, {\bf c})\bigwedge_{b\in Head_i({\bf c})\setminus L} \neg b)
\]
i.e.,
\[
\phi(I) \models  L^{\wedge} \rightarrow LF_{ORIGIN(\Pi)}(L)
\]
So $\phi(I)$ is a stable model of $SAT(\Pi)\cup ORIGIN(\Pi)$.

Next, assuming $\phi(I)$ is a stable model of $SAT(\Pi)\cup ORIGIN(\Pi)$, we will prove that $I$ belongs to $\sm[\Pi]$.
Let $\phi(I)$ be a stable model of $SAT(\Pi)\cup ORIGIN(\Pi)$. By splitting theorem, $\phi(I)$ is a stable model of $ORIGIN(\Pi)$ w.r.t. $\sigma$. So $\phi(I)$ satisfies $ORIGIN(\Pi)$. It is easy to see that this implies $I\models  \overline{\Pi_I}$. Consider any subset $L$ of $sigma$ that is satisfied by $I$. Since $\phi(I)$ is a stable model of $ORIGIN(\Pi)$, we have
\[
\phi(I) \models  LF_{ORIGIN(\Pi)}(L)
\], i.e.,
\[
\phi(I) \models  L^{\wedge} \rightarrow \bigvee_{\substack{Head_i({\bf c})\cap L\neq \emptyset, \\Head_i({\bf c}) \leftarrow Body_i({\bf c}),\\ not\ not\ {\tt sat}(i, w_i, {\bf c})\in ORIGIN(\Pi),\\ Body_i({\bf c})\cap L =\emptyset}} (Body_i({\bf c})\wedge \neg\neg {\tt sat}(i, w_i, {\bf c})\bigwedge_{b\in Head_i({\bf c})\setminus L} \neg b)
\]
All $Head_i({\bf c})\leftarrow Body_i({\bf c}), not\ not\ {\tt sat}(i, w_i, {\bf c})$ such that $Head_i({\bf c})\cap L\neq \emptyset$ are satisfied by $\phi(I)$ and thus $Head_i({\bf c})\leftarrow Body_i({\bf c})\in \overline{\Pi_I}$. So we have
\[
\phi(I) \models  L^{\wedge} \rightarrow \bigvee_{Head_i({\bf c})\cap L\neq \emptyset, Head_i({\bf c}) \leftarrow Body_i({\bf c}),\in \overline{\Pi_I}, Body_i({\bf c})\cap L =\emptyset} (Body_i({\bf c})\wedge \neg\neg {\tt sat}(i, w_i, {\bf c})\bigwedge_{b\in Head_i({\bf c})\setminus L} \neg b)
\]
Since $\phi(I)\models  {\tt sat}(i, w_i, {\bf c})$ for all $Head_i({\bf c})\cap L\neq\emptyset$ (due to that $\phi(I)\models  L$), they can be removed from the above formulas, resulting in
\[
\phi(I) \models  L^{\wedge} \rightarrow \bigvee_{Head_i({\bf c})\cap L\neq \emptyset, Head_i({\bf c}) \leftarrow Body_i({\bf c}),\in \overline{\Pi_I}, Body_i({\bf c})\cap L =\emptyset} (Body_i({\bf c})\bigwedge_{b\in Head_i({\bf c})\setminus L} \neg b)
\]
Since $\phi(I)$ and $I$ agree on all atoms in $\sigma$, it can be further rewritten as
\[
I \models  L^{\wedge} \rightarrow \bigvee_{Head_i({\bf c})\cap L\neq \emptyset, Head_i({\bf c}) \leftarrow Body_i({\bf c}),\in \overline{\Pi_I}, Body_i({\bf c})\cap L =\emptyset} (Body_i({\bf c})\bigwedge_{b\in Head_i({\bf c})\setminus L} \neg b)
\]
which means
\[
I\models  LF_{\overline{\Pi_I}}(L)
\]
So $I$ is a stable model of $\overline{\Pi_I}$, and thus is a member of $\sm[\Pi]$. 
\EOCC

\bigskip
\noindent{\bf Theorem~\ref{thm:lpmln2asp-rwd} \optional{thm:lpmln2asp-rwd}}\
\ 
{
For any $\lpmln$ program $\Pi$, there is a 1-1 correspondence $\phi$ between $\sm[\Pi]$ \footnote{Recall the definition in Section~\ref{ssec:lpmln}.}  
and the set of stable models of ${\sf lpmln2asp^{rwd}}(\Pi)$,
where 
\[ 
  \phi(I)=I\cup \{{\tt sat}(i, w_i, {\bf c}) \mid 
                   w_i:\j{Head}_i({\bf c}) \ar \j{Body}_i({\bf c})\ \text{in ${Gr}(\Pi)$}, 
                   I\models \j{Body}_i({\bf c})\rar\j{Head}_i({\bf c})\}.
\]
Furthermore,
\[ %begin{equation}\label{eq:wgt-sat}
   W_{\Pi}(I)= exp \Bigg(\sum_{{\tt sat}(i, w_i, {\bf c}) \in \phi(I)} w_i\Bigg)  .
%    \text{\ \ \ \  and\ \ \ \ } 
%% 
%   P_\Pi(I) =   
%   \underset{\alpha\to\infty}{\lim} \frac{W_\Pi(I)}{\sum_J W_\Pi(J)}
\] %end{equation}
Also, $\phi$ is a 1-1 correspondence between the most probable stable models of $\Pi$ and the optimal stable models of ${\sf lpmln2asp^{rwd}}(\Pi)$.
}
\medskip 

\begin{proof}
By Lemma~\ref{lem:lpmln2asp_rwd}, $\phi$ is a 1-1 correspondence between $\sm[\Pi]$ and the set of stable models of ${\sf lpmln2asp}^{rwd}(\Pi)$.

The fact
\begin{equation} 
   W_{\Pi}(I)= exp \Bigg(\sum_{{\tt sat}(i, w_i, {\bf c}) \in \phi(I)} w_i\Bigg)  
%    \text{\ \ \ \  and\ \ \ \ } 
%% 
%   P_\Pi(I) =   
%   \underset{\alpha\to\infty}{\lim} \frac{W_\Pi(I)}{\sum_J W_\Pi(J)}
\end{equation}
can be easily seen from the way $\phi(I)$ is defined.

It remains to show that $\phi$ is a 1-1 correspondence between the most probable stable models of~$\Pi$ and the optimal stable models of ${\sf lpmln2asp^{rwd}}(\Pi)$.
For any interpretation $I$ of ${\sf lpmln2asp^{rwd}}(\Pi)$, we use $Penalty_{\Pi}(I, l)$ to denote the total penalty it receives at level $l$ defined by weak constraints:
\[
Penalty_{\Pi}(I, l)=\sum_{\substack{:\sim {\tt sat}(i, w_i, {\bf c}).[-w^\prime_i@ l, i, {\bf c}]\in WC(\Pi),\\  I\models  {\tt sat}(i, w_i, {\bf c})}} -w_i
\]
Let $\phi(I)$ be a stable model of ${\sf lpmln2asp^{rwd}}(\Pi)$. By Lemma \ref{lem:lpmln2asp_rwd}, $I\in \sm[\Pi]$. So it is sufficient to prove
\begin{align}
\nonumber
&I\in \underset{J:J\in \underset{K:K\in \sm[\Pi]}{\argmax} W_{\Pi^{\rm hard}}(K)}{\argmax} W_{\Pi^{\rm soft}}(J)  \\
\text{ iff} \\
&\nonumber \phi(I)\in  \underset{J^\prime:J^\prime\in\underset{K^\prime: \substack{\text{$K^\prime$ is a stable model of}\\ {\sf lpmln2asp^{\rm rwd}}(\Pi)}}{\argmin} Penalty_{\sf lpmln2asp^{rwd}(\Pi)}(K^\prime, 1)}{\argmin} Penalty_{\sf lpmln2asp^{rwd}(\Pi)}(J^\prime, 0).
\end{align}

This is true because (we abbreviate $Head_i({\bf c})\leftarrow Body_i({\bf c})$ as $F_i({\bf c})$)
\[
\ba {l l}
I\in & \argmax_{J:~J \in \argmax_{K:~K \in {\rm SM}[\Pi]} W_{\Pi^{\rm hard}}(K)} W_{\Pi^{\rm soft}}(J) \\

\text{iff} &   \\

I\in &\argmax_{J:~J \in \argmax_{K:~K \in {\rm SM}[\Pi]} exp\big(\sum\limits_{\alpha:F_i({\bf c})\;\in\; (\Pi^{\rm hard})_{K}} \alpha\big)} exp\Big(\sum\limits_{w_i:F_i({\bf c})\;\in\; (\Pi^{\rm soft})_{J}} w_i\Big)\\

\text{iff} & \\

I\in & \argmax_{J:~J \in \argmax_{K:~K \in {\rm SM}[\Pi]} exp\big(\sum\limits_{\alpha:F_i({\bf c})\;\in\; \Pi^{\rm hard}, K\models  F_i({\bf c})} 1\big)} exp\Big(\sum\limits_{w_i:F_i({\bf c})\;\in\; \Pi^{\rm soft}, J\models  F_i({\bf c})} w_i\Big)\\

\text{iff} & \\

I\in & \argmin_{J:~J \in \argmin_{K:~K \in {\rm SM}[\Pi]} \big(\sum\limits_{\alpha:F_i({\bf c})\;\in\; \Pi^{\rm hard}, K\models  F_i({\bf c})} -1\big)} \Big(\sum\limits_{w_i:F_i({\bf c})\;\in\; \Pi^{\rm soft}, J\models  F_i({\bf c})} -w_i\Big)\\

\text{iff} & \text{(by Lemma \ref{lem:lpmln2asp_rwd} and by definition of $\phi(I)$)}\\

\phi(I)\in & \argmin_{J':~J' \in \argmin_{\substack{K':~K' \text{ is a stable model of } \\{\sf lpmln2asp^{\rm rwd}}(\Pi)}} \bigg(\sum\limits_{\substack{:\sim {\tt sat}(i, w_i, {\bf c}).[-1@ 1, i, {\bf c}] \\\in\ {\sf lpmln2asp^{\rm rwd}}(\Pi),\\ K'\models  {\tt sat}(i, w_i, {\bf c})}} -1\bigg)} \bigg(\sum\limits_{\substack{:\sim {\tt sat}(i, w_i, {\bf c}).[-w_i@ 0, i, {\bf c}]\\ \in\ {\sf lpmln2asp^{\rm rwd}}(\Pi),\\ J'\models  {\tt sat}(i, w_i, {\bf c})}} -w_i\bigg)\\

\text{iff} & \\

\phi(I)\in & \argmin_{J':~J' \in \argmin_{\substack{K':~K' \text{ is a stable model of } \\{\sf lpmln2asp^{\rm rwd}}(\Pi)}} Penalty_{{\sf lpmln2asp^{\rm rwd}}(\Pi)}(K',1)} Penalty_{{\sf lpmln2asp^{\rm rwd}}(\Pi)}(J',0).
\ea
\]

\end{proof}

\BOCCC
%------------------------------------------------------------------
\section{Proof of Proposition \ref{prop:clingo-simplification}} 
%------------------------------------------------------------------
\noindent{\bf Proposition~\ref{prop:clingo-simplification} \optional{prop:clingo-simplification}}\
\ 
{\sl
The probability of any interpretation $I$ defined by (\ref{eq:wgt-unsat}) remains the same  if we replace ${\sf lpmln2asp^{pnt}}(\Pi)$ with ${\sf lpmln2asp^{pnt}}(\Pi)\setminus AllTrue$.
\BOCC
The probability of any interpretation $I$ computed by (\ref{eq:wgt-unsat}) does not change if we replace $lpmln2asp^{pnt}(\Pi)$ with $lpmln2asp^{pnt}(\Pi)\setminus REDUNDT$, (with all the unsat atoms that occurs in $REDUNDT$ only set to {\sc false}),  more specifically,
\[
exp \Bigg(-\sum_{\substack{{\tt unsat}(i, w_i, {\bf c}) \in \phi(I)\\\text{${\tt unsat}(i, w_i, {\bf c})$ occurs in $lpmln2asp^{pnt}(\Pi)$}}} w_i\Bigg) = exp \Bigg(-\sum_{\substack{{\tt unsat}(i, w_i, {\bf c}) \in \phi(I)\\\text{${\tt unsat}(i, w_i, {\bf c})$ occurs in $lpmln2asp^{pnt}(\Pi)\setminus REDUNDT$}}} w_i\Bigg).
\]
\EOCC

}

\BOCC
\begin{proof}
For any stable model $I$ of $lpmln2asp^{pnt}(\Pi)$, consider the following two sets
\begin{align}
\label{eq:unsat-set1} &\{{\tt unsat}(i, w_i, {\bf c}) \mid {\tt unsat}(i, w_i, {\bf c})\notin I, {\text{${\tt unsat}(i, w_i, {\bf c})$ occurs in $lpmln2asp^{pnt}(\Pi)\setminus REDUNDT$}} \}\\
\nonumber &\text{and}\\
\label{eq:unsat-set2} &\{{\tt unsat}(i, w_i, {\bf c}) \mid {\tt unsat}(i, w_i, {\bf c})\notin I, {\text{${\tt unsat}(i, w_i, {\bf c})$ occurs in $lpmln2asp^{pnt}(\Pi)$}} \}
\end{align}
Clearly, we have
\begin{align}
\nonumber &\{{\tt unsat}(i, w_i, {\bf c}) \mid {\tt unsat}(i, w_i, {\bf c})\notin I, {\text{${\tt unsat}(i, w_i, {\bf c})$ occurs in $lpmln2asp^{pnt}(\Pi)\setminus REDUNDT$}} \}\\
\nonumber &\subseteq\\
\nonumber &\{{\tt unsat}(i, w_i, {\bf c}) \mid {\tt unsat}(i, w_i, {\bf c})\notin I, {\text{${\tt unsat}(i, w_i, {\bf c})$ occurs in $lpmln2asp^{pnt}(\Pi)$}} \}
\end{align}
Suppose ${\tt unsat}(i, w_i, {\bf c})$ is in (\ref{eq:unsat-set2}) but not in (\ref{eq:unsat-set1}). This mean the three rules where ${\tt unsat}(i, w_i, {\bf c})$ occurs
\[
\ba {rcl}
  {\tt unsat}(i, w_i, {\bf c}) & \ar & Body_i, {\tt not}\ Head_i\\
  Head_i  & \ar &  Body_i, {\tt not}\ {\tt unsat}(i, w_i, {\bf c})\\
           &:\sim & {\tt unsat}(i, w_i, {\bf x}). \ \ \ [w_i@l, i, {\bf c}]
\ea
\]
are in $REDUNDT$, which further implies that every stable model $J$ of $lpmln2asp^{pnt}(\Pi)$ satisfies the three rules. Since ${\tt unsat}(i, w_i, {\bf c})\notin I$, the only possiblility is $J\not\models   Body_i, {\tt not}\ Head_i$, i.e., $J$ satisfies $R_I$. Consequently, we have
\begin{align}
\nonumber 
&\{{\tt unsat}(i, w_i, {\bf c}) \mid {\tt unsat}(i, w_i, {\bf c})\notin I, {\text{${\tt unsat}(i, w_i, {\bf c})$ occurs in $lpmln2asp^{pnt}(\Pi)$}} \}\\
\nonumber 
&=\\
\nonumber
&\{{\tt unsat}(i, w_i, {\bf c}) \mid {\tt unsat}(i, w_i, {\bf c})\notin I, {\text{${\tt unsat}(i, w_i, {\bf c})$ occurs in $lpmln2asp^{pnt}(\Pi)\setminus REDUNDT$}} \}\cup\\
\nonumber &\{R_i\in\Pi\mid\text{$R_i$ is satisfied by all stable model of $\Pi$ and $unsat(i,w_i,{\bf c})$ only occurs in $REDUNDT$}\}
\end{align}
Clearly,
\[
\{{\tt unsat}(i, w_i, {\bf c}) \mid {\tt unsat}(i, w_i, {\bf c})\notin I, {\text{${\tt unsat}(i, w_i, {\bf c})$ occurs in $lpmln2asp^{pnt}(\Pi)\setminus REDUNDT$}} \}
\]
and 
\[
\{R_i\in\Pi\mid\text{$R_i$ is satisfied by all stable model of $\Pi$ and $unsat(i,w_i,{\bf c})$ only occurs in $REDUNDT$}\}
\]
has empty intersection. So we have
\[
exp \Bigg(\sum_{\substack{{\tt unsat}(i, w_i, {\bf c}) \notin I\\\text{${\tt unsat}(i, w_i, {\bf c})$ occurs in $lpmln2asp^{pnt}(\Pi)$}}} w_i\Bigg) = k\cdot exp \Bigg(\sum_{\substack{{\tt unsat}(i, w_i, {\bf c}) \notin I\\\text{${\tt unsat}(i, w_i, {\bf c})$ occurs in $lpmln2asp^{pnt}(\Pi)\setminus REDUNDT$}}} w_i\Bigg)
\]
\end{proof}
\EOCC

\begin{proof}
For any stable model $I$ of $lpmln2asp^{pnt}(\Pi)$, consider the following two sets
\begin{align}
\label{eq:unsat-set1} &\{{\tt unsat}(i, w_i, {\bf c}) \mid {\tt unsat}(i, w_i, {\bf c})\in I, {\text{${\tt unsat}(i, w_i, {\bf c})$ occurs in $lpmln2asp^{pnt}(\Pi)\setminus REDUNDT$}} \}\\
\nonumber &\text{and}\\
\label{eq:unsat-set2} &\{{\tt unsat}(i, w_i, {\bf c}) \mid {\tt unsat}(i, w_i, {\bf c})\in I, {\text{${\tt unsat}(i, w_i, {\bf c})$ occurs in $lpmln2asp^{pnt}(\Pi)$}} \}
\end{align}
Clearly, we have
\begin{align}
\nonumber &\{{\tt unsat}(i, w_i, {\bf c}) \mid {\tt unsat}(i, w_i, {\bf c})\in I, {\text{${\tt unsat}(i, w_i, {\bf c})$ occurs in $lpmln2asp^{pnt}(\Pi)\setminus REDUNDT$}} \}\\
\nonumber &\subseteq\\
\nonumber &\{{\tt unsat}(i, w_i, {\bf c}) \mid {\tt unsat}(i, w_i, {\bf c})\in I, {\text{${\tt unsat}(i, w_i, {\bf c})$ occurs in $lpmln2asp^{pnt}(\Pi)$}} \}
\end{align}
Suppose ${\tt unsat}(i, w_i, {\bf c})$ is in (\ref{eq:unsat-set2}) but not in (\ref{eq:unsat-set1}). This mean the three rules where ${\tt unsat}(i, w_i, {\bf c})$ occurs
\[
\ba {rcl}
  {\tt unsat}(i, w_i, {\bf c}) & \ar & Body_i, {\tt not}\ Head_i\\
  Head_i  & \ar &  Body_i, {\tt not}\ {\tt unsat}(i, w_i, {\bf c})\\
           &:\sim & {\tt unsat}(i, w_i, {\bf x}). \ \ \ [w_i@l, i, {\bf c}]
\ea
\]
are in $REDUNDT$, which further implies that every stable model $J$ of $lpmln2asp^{pnt}(\Pi)$ satisfies the three rules, and thus ${\tt unsat}\notin J$ (due to the last rule). This contradicts the fact that ${\tt unsat}\in I$. So there does not exist any ${\tt unsat}(i, w_i, {\bf c})$ that is in (\ref{eq:unsat-set2}) but not in (\ref{eq:unsat-set1}), and consequently,
\begin{align}
\nonumber &\{{\tt unsat}(i, w_i, {\bf c}) \mid {\tt unsat}(i, w_i, {\bf c})\in I, {\text{${\tt unsat}(i, w_i, {\bf c})$ occurs in $lpmln2asp^{pnt}(\Pi)\setminus REDUNDT$}} \}\\
\nonumber &=\\
\nonumber &\{{\tt unsat}(i, w_i, {\bf c}) \mid {\tt unsat}(i, w_i, {\bf c})\in I, {\text{${\tt unsat}(i, w_i, {\bf c})$ occurs in $lpmln2asp^{pnt}(\Pi)$}} \}.
\end{align}
So
\[
exp \Bigg(-\sum_{\substack{{\tt unsat}(i, w_i, {\bf c}) \in I\\\text{${\tt unsat}(i, w_i, {\bf c})$ occurs in $lpmln2asp^{pnt}(\Pi)$}}} w_i\Bigg) = exp \Bigg(-\sum_{\substack{{\tt unsat}(i, w_i, {\bf c}) \in I\\\text{${\tt unsat}(i, w_i, {\bf c})$ occurs in $lpmln2asp^{pnt}(\Pi)\setminus REDUNDT$}}} w_i\Bigg).
\]
\end{proof}
\EOCCC

%------------------------------------------------------------------
\section{Proof of Proposition~\ref{prop:mln-aux}} 
%------------------------------------------------------------------

For any MLN $\ML$ and any interpretation $I$, we define
\[
W^\prime_{\ML}(I) =
\begin{cases}
exp(\sum_{w:F\in\ML^{soft},I\vDash F)}w) & \text{if $I\vDash \overline{\ML^{hard}}$}\\
0 & \text{otherwise.}
\end{cases}
\]
\begin{lemma}\label{lem:mln-soft}
For any MLN $\ML$ such that $\overline{\ML^{hard}}$ has at least one model, we have
\[
P_{\ML}(I) = \frac{W^\prime_{\ML}(I)}{\sum_{J}W^\prime_{\ML}(J)}
\]
for any interpretation $I$.
\end{lemma}

\begin{proof}
\noindent
{\bf Case 1}: Suppose $I\vDash \overline{\ML^{hard}}$.
\begin{align}
\nonumber P_{\ML}(I) &= \underset{\alpha\to\infty}{lim}\frac{W_{\ML}(I)}{\sum_{J}W_{\ML}(J)}\\
\nonumber &= \underset{\alpha\to\infty}{lim}\frac{exp(\sum_{w:F\in \ML, I\vDash F}w)}{\sum_J exp(\sum_{w:F\in \ML, J\vDash F}w)}\\
\nonumber &=  \underset{\alpha\to\infty}{lim}\frac{exp(|\overline{\ML^{hard}}|\alpha)\cdot exp(\sum_{w:F\in \ML\setminus\ML^{hard}, I\vDash F}w)}{\sum_{J\vDash \overline{\ML^{hard}}}exp(|\overline{\ML^{hard}}|\alpha)\cdot exp(\sum_{w:F\in \ML\setminus\ML^{hard}, I\vDash F}w) + \sum_{J\nvDash \overline{\ML^{hard}}}exp(\sum_{w:F\in \ML, I\vDash F}w)}\\
\nonumber &= \underset{\alpha\to\infty}{lim}\frac{exp(\sum_{w:F\in \ML\setminus\ML^{hard}, I\vDash F}w)}{\sum_{J\vDash \overline{\ML^{hard}}}exp(\sum_{w:F\in \ML\setminus\ML^{hard}, I\vDash F}w) + \frac{1}{exp(|\overline{\ML^{hard}}|\alpha)} \sum_{J\nvDash \overline{\ML^{hard}}}exp(\sum_{w:F\in \ML, I\vDash F}w)}.
\end{align}
Since there is at least one hard formula in $\overline{\ML^{hard}}$ not satisfied by those $J$ that do not satisfy $\overline{\ML^{hard}}$, we have
\begin{align}
\nonumber &\frac{1}{exp(|\overline{\ML^{hard}}|\alpha)} \sum_{J\nvDash \overline{\ML^{hard}}}exp(\sum_{w:F\in \ML, I\vDash F}w) \\
\nonumber \leq &\frac{1}{exp(|\overline{\ML^{hard}}|\alpha)} \sum_{J\nvDash \overline{\ML^{hard}}}exp((|\overline{\ML^{hard}}|-1)\alpha + \sum_{w:F\in \ML\setminus\ML^{hard}, I\vDash F}w)\\
\nonumber &= \frac{1}{exp(\alpha)} \sum_{J\nvDash \overline{\ML^{hard}}}exp(\sum_{w:F\in \ML\setminus\ML^{hard}, I\vDash F}w).
\end{align}
This, along with the fact that $\sum_{J\nvDash \overline{\ML^{hard}}}exp(\sum_{w:F\in \ML\setminus\ML^{hard}, I\vDash F}w)$ does not contain $\alpha$, we have
\begin{align}
\nonumber P_{\ML}(I) &= \underset{\alpha\to\infty}{lim}\frac{exp(\sum_{w:F\in \ML\setminus\ML^{hard}, I\vDash F}w)}{\sum_{J\vDash \overline{\ML^{hard}}}exp(\sum_{w:F\in \ML\setminus\ML^{hard}, I\vDash F}w) + \frac{1}{exp(|\overline{\ML^{hard}}|\alpha)} \sum_{J\nvDash \overline{\ML^{hard}}}exp(\sum_{w:F\in \ML, I\vDash F}w)}\\
\nonumber &= \frac{exp(\sum_{w:F\in \ML\setminus\ML^{hard}, I\vDash F}w)}{\sum_{J\vDash \overline{\ML^{hard}}}exp(\sum_{w:F\in \ML\setminus\ML^{hard}, I\vDash F}w)}\\
\nonumber &= \frac{exp(\sum_{w:F\in \ML^{soft}, I\vDash F}w)}{\sum_{J\vDash \overline{\ML^{hard}}}exp(\sum_{w:F\in \ML^{soft}, I\vDash F}w)}\\
\nonumber &= \frac{W^\prime_{\ML}(I)}{\sum_{J}W^\prime_{\ML}(J)}.
\end{align}

\noindent
{\bf Case 2}: Suppose $I$ does not satisfy $\overline{\ML^{hard}}$. Let $K$ be an interpretation that satisfies $\overline{\ML^{hard}}$. We have
\begin{align}
\nonumber P_{\ML}(I) &= \underset{\alpha\to\infty}{lim}\frac{W_{\ML}(I)}{\sum_{J}W_{\ML}(J)}\\
\nonumber &= \underset{\alpha\to\infty}{lim}\frac{exp(\sum_{w:F\in \ML, I\vDash F}w)}{\sum_J exp(\sum_{w:F\in \ML, J\vDash F}w)}\\
\nonumber & \leq \underset{\alpha\to\infty}{lim}\frac{exp(\sum_{w:F\in \ML, I\vDash F}w)}{exp(\sum_{w:F\in \ML, K\vDash F}w)}\\
\nonumber  & = \underset{\alpha\to\infty}{lim}\frac{exp(\sum_{w:F\in \ML, I\vDash F}w)}{exp(|\overline{\ML^{hard}}|\alpha)\cdot exp(\sum_{w:F\in \ML^{soft}, K\vDash F}w)}.
\end{align}
Since $I$ does not satisfy $\overline{\ML^{hard}}$, $I$ satisfies at most $\overline{\ML^{hard}} - 1$ hard formulas in $\overline{\ML^{hard}}$. So we have
\begin{align}
\nonumber P_{\ML}(I) &\leq\underset{\alpha\to\infty}{lim}\frac{exp(\sum_{w:F\in \ML, I\vDash F}w)}{exp(|\overline{\ML^{hard}}|\alpha)\cdot exp(\sum_{w:F\in \ML^{soft}, K\vDash F}w)}\\
\nonumber \leq & \underset{\alpha\to\infty}{lim}\frac{(exp(|\overline{\ML^{hard}} |-1)\alpha)\cdot exp(\sum_{w:F\in \ML^{soft}, I\vDash F}w)}{(exp(|\overline{\ML^{hard}}|\alpha )\cdot exp(\sum_{w:F\in \ML^{soft}, K\vDash F}w)}\\
\nonumber = & \underset{\alpha\to\infty}{lim}\frac{exp(\sum_{w:F\in \ML^{soft}, I\vDash F}w)}{exp(\alpha)\cdot exp(\sum_{w:F\in \ML^{soft}, K\vDash F}w)}\\
\nonumber = &\ 0.
\end{align}
So $ P_{\ML}(I) = 0$, which is equivalent to $\frac{W^\prime_{\ML}(I)}{\sum_{J}W^\prime_{\ML}(J)}$ as $W^\prime_{\ML}(I) = 0$.
\end{proof}

\BOCC
\begin{lemma}
\label{lem:mln-aux-ground}
Given any MLN program $\Pi$, let $F$ be a subformula that occurs in some formula in $\Pi$, and $\Pi^{Aux}_F$ be the MLN program obtained from $\Pi$ by replacing $F$ with $Aux_F$ and adding the formula
\[
\nonumber \alpha\ \ :\ \ Aux_F \leftrightarrow F.
\]
For any interpretation $I$ of $\Pi$, let $I^{Aux}_F$ be the interpretation defined as 
\begin{itemize}
\item $I^{Aux}_F(p)=I(p)$ for all $p$ that occurs in $\Pi$;
\item $I^{Aux}_F(Aux_F)=I(F)$.
\end{itemize}

We have
\[
P_{\Pi}(I) = P_{\Pi^{Aux}_F}(I^{Aux}_F)
\]
for any interpretation $I$.
\end{lemma}

\EOCC

\bigskip
\noindent{\bf Proposition~\ref{prop:mln-aux} \optional{prop:mln-aux}}\
\ 
{
For any MLN $\ML$ of signature $\sigma$, let $F({\bf x})$ be a subformula of some formula in $\ML$ where ${\bf x}$ is the list of all free variables of $F({\bf x})$, and let $\ML^F_{Aux}$ be the MLN program obtained from $\ML$ by replacing $F({\bf x})$ with a new predicate $Aux({\bf x})$ and adding the formula
\[
  \alpha\ \ :\ \  Aux({\bf x}) \leftrightarrow F({\bf x}).
\]
For any interpretation $I$ of $\ML$, let $I_{Aux}$ be the extension of $I$ of signature $\sigma\cup \{Aux\}$ defined by $I_{Aux}(Aux({\bf c}))=(F({\bf c}))^I$ for every list ${\bf c}$ of elements in the Herbrand universe.
%\begin{itemize}
%\item $I^F_{Aux}(p)=I(p)$ for all $p$ that occurs in $\Pi$;
%\item $I^F_{Aux}(Aux)=I(F)$.
%\end{itemize}
%
When $\overline{\ML^{hard}}$ has at least one model, we have 
\[
    P_{\ML}(I) = P_{\ML^F_{Aux}}(I_{Aux}).
\]
}

\begin{proof}
For any formula $G$, let $G^F_{Aux}$ be the formulas obtained from $G$ by replacing subformulas $F({\bf x})$ with $Aux_F({\bf x})$.
According to Lemma \ref{lem:mln-soft}, we have
\begin{align}
\nonumber P_{\ML}(I) &= \frac{W^\prime_{\ML}(I)}{\sum_{J}W^\prime_{\ML}(J)}.
\end{align}
{\bf Case 1:} Suppose $I$ satisfies $\overline{\ML^{hard}}$. Then we have
\begin{align}
\nonumber P_{\ML}(I) &= \frac{exp(\sum_{w:G\in\ML^{soft},I\vDash F}w)}{\sum_{J\vDash\overline{\ML^{hard}}}exp(\sum_{w:G\in\ML^{soft},J\vDash F}w)}.
\end{align}

From the way $I_{Aux}$ is defined, we have
\begin{align}
\nonumber P_{\ML}(I) &= \frac{exp(\sum_{w:G^{F}_{Aux}\in(\ML^{F}_{Aux})^{soft},I_{Aux}\vDash G^{F}_{Aux}}w)}{\sum_{J_{Aux}\vDash\overline{(\ML^{F}_{Aux})^{hard}}}exp(\sum_{w:G^{F}_{Aux}\in(\ML^{F}_{Aux})^{soft},J_{Aux}\vDash G^{F}_{Aux}}w)}\\
\nonumber &= P_{\ML^{F}_{Aux}}(I_{Aux}).
\end{align}

{\bf Case 2:} Suppose $I$ does not satisfy $G\in\overline{\ML^{hard}}$. From the way $I_{Aux}$ is defined, $I_{Aux}$ does not satisfy $G^{F}_{Aux}\in\overline{(\ML^{F}_{Aux})^{hard}}$. So $W^\prime_{\ML}(I) = W^\prime_{\ML^{F}_{Aux}}(I_{Aux}) = 0$ and thus $P_{\ML}(I) = P_{\ML^{F}_{Aux}}(I_{Aux}) = 0$.

\BOCC
For any formula $G$, let $G^{Aux}_F$ be the formulas obtained from $G$ by replacing subformulas $F({\bf x})$ with $Aux_F({\bf x})$. For every list ${\bf c}$ of element in the Herbrand universe, since $I^{Aux}_F(Aux_F({\bf c}))$ is defined as $(F({\bf c}))^I$, $I^{Aux}_F$ satisfies $Aux_F({\bf c}) \leftrightarrow F({\bf c})$. For a formula $G_i$, let $n_i$ be the number of its ground instances. Let $k$ denote the number of ground instances of $Aux_F({\bf c}) \leftrightarrow F({\bf c})$.
\begin{align}
\nonumber P_{\ML^{Aux}_{F}}(I^{Aux}_F) &= \underset{\alpha\to\infty}{lim} \frac{exp(\sum_{w_i:G_i\in \ML^{Aux}_F}n_iw_i)}{\sum_{J}exp(\sum_{w_i:G_i\in \ML^{Aux}_F}n_iw_i)}\\
\nonumber &= \underset{\alpha\to\infty}{lim}exp\frac{exp(k\alpha)\cdot exp(\sum_{w_i:G_i\in \ML}n_iw_i)}{\sum_{J\models  \forall {\bf x}(Aux_F({\bf x}) \leftrightarrow F({\bf x}))}exp(\sum_{w_i:G_i\in \ML^{Aux}_F}n_iw_i) + \sum_{J\not\models  \forall {\bf x} (Aux_F({\bf x}) \leftrightarrow F(\bf x))}exp(\sum_{w_i:G_i\in \ML^{Aux}_F}n_iw_i)}\\
\nonumber &= \underset{\alpha\to\infty}{lim}\frac{exp(\sum_{w_i:G_i\in \ML}n_iw_i)}{\sum_{J\models  \forall {\bf x}(Aux_F({\bf x}) \leftrightarrow F({\bf x}))}exp(\sum_{w_i:G_i\in \ML}n_iw_i) + \frac{1}{exp(k\alpha)}\sum_{J\not\models  \forall {\bf x}(Aux_F({\bf x}) \leftrightarrow F({\bf x}))}exp(\sum_{w_i:G_i\in \ML^{Aux}_F}n_iw_i)}
\end{align}
Note that the second term of the denominator
\begin{align}
\nonumber & \frac{1}{exp(k\alpha)}\sum_{J\not\models  \forall {\bf x}(Aux_F({\bf x}) \leftrightarrow F({\bf x}))}exp(\sum_{w_i:G_i\in \ML^{Aux}_F}n_iw_i) \\
\nonumber = &\sum_{J\not\models  \forall {\bf x}(Aux_F({\bf x}) \leftrightarrow F({\bf x}))}exp(\sum_{w_i:G_i\in \ML^{Aux}_F}n_iw_i - k\alpha) 
\end{align}
Since for those $J$ that do not satisfy $\forall {\bf x}(Aux_F({\bf x}) \leftrightarrow F({\bf x}))$, there is at least one ground instance of $Aux_F({\bf x}) \leftrightarrow F({\bf x})$ not satisfied by $J$, we have
\begin{align}
\nonumber &\sum_{J\not\models  \forall {\bf x}(Aux_F({\bf x}) \leftrightarrow F({\bf x}))}exp(\sum_{w_i:G_i\in \ML^{Aux}_F}n_iw_i - k\alpha)\\
\nonumber \leq&\sum_{J\not\models  \forall {\bf x}(Aux_F({\bf x}) \leftrightarrow F({\bf x}))}exp(\sum_{w_i:G_i\in \ML}n_iw_i + (k-1)\alpha - k\alpha)\\
\nonumber = &\sum_{J\not\models  \forall {\bf x}(Aux_F({\bf x}) \leftrightarrow F({\bf x}))}exp(\sum_{w_i:G_i\in \ML}n_iw_i - \alpha)\\
\nonumber = &\frac{1}{\alpha}\sum_{J\not\models  \forall {\bf x}(Aux_F({\bf x}) \leftrightarrow F({\bf x}))}exp(\sum_{w_i:G_i\in \ML}n_iw_i )
\end{align}
So
\begin{align}
\nonumber P_{\ML^{Aux}_{F}}(I^{Aux}_F) &= \underset{\alpha\to\infty}{lim}\frac{exp(\sum_{w_i:G_i\in \ML}n_iw_i)}{\sum_{J\models \forall {\bf x} Aux_F({\bf x}) \leftrightarrow F({\bf x})}exp(\sum_{w_i:G_i\in \ML}n_iw_i)}\\
\nonumber &= \underset{\alpha\to\infty}{lim}\frac{exp(\sum_{w_i:G_i\in \ML}n_iw_i)}{\sum_{J}exp(\sum_{ w_i:G_i\in\ML}n_iw_i)}\\
\nonumber &=P_{\ML}(I).
\end{align}
\EOCC
\end{proof}

%-----------------------------------------------------------------------------------------------------
\section{More Experiments}
%-----------------------------------------------------------------------------------------------------

\BOCC
%-----------------------------------------------------------------------------------------------------
\subsection{Maximal Relaxed Clique}
%-----------------------------------------------------------------------------------------------------

We experiment on the problem of finding a maximal relaxed clique in a graph.  The goal is, given a graph of connected nodes, select as many nodes as possible in the graph to create a subgraph. In the subgraph, we assign a reward to every pair of connected nodes and a reward for every node included in the subgraph. The reward of the subgraph determine how much ``relaxed'' the clique is. A \emph{maximal relaxed clique} is a subgraph that maximizes the reward that can be given to a subgraph.

The \la encoding of the above problem is
\begin{lstlisting}
{in(X)} :- node(X).
disconnected(X, Y) :- in(X), in(Y), not edge(X, Y).
5  :- not in(X), node(X).
5  :- disconnected(X, Y).
\end{lstlisting}

Rule 1 states that every node$X$ can be {\bf in} the subgraph. Rule 2 states that a pair of {\bf nodes} is disconnected if it is {\bf not in} the subgraph and there is no {\bf edge} between those two. Rule 3 states that If a node is {\bf in} the subgraph, we give it a reward of 5. Given an interpretation $I$, a pair of nodes can be {\bf disconnected} if two nodes $X$ and $Y$ are {\bf in} $I$ and there is not edge between them. Rule 4 states that if two nodes $X_i$ and $Y_i$ are not {\bf disconnected} in an interpretation $I$ we give a reward of 5 to $I$. 

Similarly the \lm encoding of the above problem is
\begin{lstlisting}
NodeSet = {1}
In(NodeSet)
Node(NodeSet)
Edge(NodeSet, NodeSet)
Disconnected(NodeSet, NodeSet)

{In(x)} <= Node(x).
Disconnected(x, y) <= In(x) ^ In(y) ^ !Edge(x, y).
5  <= !In(x) ^ Node(x)
5  <= Disconnected(x, y)
\end{lstlisting}
We declare the sort \lstinline|NodeSet| containing just $1$ node. Additional nodes can be present in the evidence file.

\bex {Consider the graph and its \la encoding as given above.
\begin{figure}[h!]
	\centering
	\includegraphics{mrc-ex.png}
	\label{fig:mrc-ex}
	\caption{Maximal Relaxed Clique Example}
\end{figure}
}

Consider an interpretation $I = \{in(1), in(2), in(3)\}$. For this interpretation, the reward under the \lp semantics is 95 given by $(w_3 * 3) + (w_4 * 16) = 95$ where $w_i$ represents the weight of the $t^{th}$ rule in the \la encoding. Consider an interpretation $I = \{in(1), in(2), in(3), in(4)\}$. For this interpretation, the reward under the \lp semantics is $(w_3 * 4) + (w_4 * 14) = 90$. The reward in the latter case is lesser even though all nodes are included since it has two pair of disconnected nodes: $(1,4)$ and $(4,1)$. Since the interpretation $I = \{in(1), in(2), in(3)\}$ results in the maximum reward, $I$ is the maximal relaxed clique. Note that $I = \{in(2), in(3), in(4)\}$ is another maximal relaxed clique of the same graph.

For this experiment, we generate a graph by randomly generating edges between nodes with probability $\{0.5,  0.8,  0.9,  1\}$ and different number of nodes $\{10,20,50,100,200,300,$ $400,500\}$ at each probability. For each problem instance, we perform MAP inferences to find maximal relaxed cliques with both {\sc lpmln2asp} and {\sc lpmln2mln}. The timeout is set to 20 minutes. The experiments are performed on a machine powered by 4 Intel(R) Core(TM) i5-2400 CPU with OS Ubuntu 14.04 LTS and 8G memory.

\renewcommand{\arraystretch}{0.8}
%\begin{table}[p]
%	{\footnotesize
%		\centering
%			\thispagestyle{empty}
\begin{center}
		\begin{longtable}{| c | c | c | c | c |}
			\hline
%			{\bf Parameter} & \begin{tabular}{@{}c@{}}{\la \\ {\bf w.} \cli \ {\bf 4.5}}\end{tabular} & \begin{tabular}{@{}c@{}}{\lm \\ {\bf w.} \al \ {\bf 2}}\end{tabular} & \begin{tabular}{@{}c@{}}{\lm \\ {\bf w.} \tu \ {\bf 0.3}}\end{tabular} & \begin{tabular}{@{}c@{}}{\lm \\ {\bf w.} \ro \ {\bf 0.5}}\end{tabular}\\
	{\bf Instance} & \makecell{\la \\ {\bf w.} \cli \ {\bf 4.5}} & \makecell{\lm \\ {\bf w.} \al \ {\bf 2}} & \makecell{\lm \\ {\bf w.} \tu \ {\bf 0.3}} & \makecell{\lm \\ {\bf w.} \ro \ {\bf 0.5}}\\
			\hline
			\hline
				p50n10\tablefootnote{pXnY is a graph with probability X and \# of nodes Y} & \btc{0.004\tablefootnote{Grounding time of clingo is not reported since it is negligible}}{205}{239} & \btm{0.11}{2.84\tablefootnote{grounding time + solving time}}{310}{788\tablefootnote{\# of atoms/\# of clauses in \al}} & \btm{3}{6.083\tablefootnote{grounding time + solving time}}{310}{786\tablefootnote{\# of atoms/\# of clauses in \tu}} & \btr{0.693\tablefootnote{grounding time + solving time}}{71\tablefootnote{\# of atoms}}\\
			\hline
			p50n20 & \btc{0.017}{769}{887} & \btm{0.34}{4.85}{1,220}{3,136} & \btm{3}{27.345}{1,220}{3,132} & \btr{0.884}{242}\\
			\hline
			p50n50 & \btc{Timeout}{4,511}{5,185} & \btm{2.5}{11.18}{7,550}{19,466} & \btm{3}{Timeout}{7,550}{19,438} & \btr{1.046}{1320}\\
			\hline
			p50n100 & \btc{Timeout}{17,744}{20,294} & \btm{17.01}{36.88}{30,100}{77,643} & \btm{5}{Timeout}{30,100}{77,593} & \btr{1.539}{5229}\\
			\hline
			p50n200 & \btc{Timeout}{70,570}{80,696} & \btm{154.75}{235.62}{120,200}{310,395} & \btm{13}{Timeout}{120,200}{310,382} & \btr{6.903}{20,376}\\
			\hline
			p50n300 & \btc{Timeout}{135,924}{181,247} & \btm{564.62}{799.61}{270,300}{698,022} & \btm{36}{Timeout}{270,300}{697,868} & \btr{20.239}{45,501}\\
			\hline
			p50n400 & \btc{Timeout}{241,400}{321,999} & Timeout & Timeout & \btr{Timeout}{80,472}\\
			\hline
			p50n500 & \btc{Timeout}{376,7332}{502,465} & Timeout & Timeout & Timeout\\
			\hline
			\hline
			p80n10 & \btc{0.004}{219}{235} & \btm{0.1}{2.27}{310}{802} & \btm{3}{5.568}{310}{800}&\btr{0.755}{102}\\
			\hline
			p80n20 & \btc{0.014}{802}{808} & \btm{0.31}{3.4}{1,220}{3,169} & \btm{3}{19.001}{1,220}{3,165} & \btr{0.759}{366}\\
			\hline
			p80n50 & \btc{2.730}{4,739}{4,785} & \btm{2.45}{12.2}{7,550}{19,658} & \btm{3}{712.472}{7,550}{19,648} & \btr{1.11}{2,099}\\
			\hline
			p80n100 & \btc{Timeout}{18,717}{19,193} & \btm{17.35}{42.11}{30,100}{78,552} & \btm{5}{Timeout}{30,100}{78,534} & \btr{1.941}{8,261}\\
			\hline
			p80n200 & \btc{Timeout}{74,174}{75,913} & \btm{157.82}{261.28}{120,200}{313,845} & \btm{12}{Timeout}{120,200}{313,809} & \btr{179.944}{32,400}\\
			\hline
			p80n300 & \btc{Timeout}{108,969}{127,337} & \btm{564.6}{878.63}{270,300}{705,886} & \btm{30}{Timeout}{270,300}{705,831} & \btr{Timeout}{72,427}\\
			\hline
			p80n400 & \btc{Timeout}{193,151}{225,501} & Timeout & Timeout & \btr{Timeout}{129,097}\\
			\hline
			p80n500 & \btc{Timeout}{461,326}{471,714} & Timeout & Timeout & \btr{Timeout}{201,312}\\
 			\hline
			\hline
			p90n10 & \btc{0.004}{223}{233} & \btm{0.06}{0.07}{310}{802} & \btm{3}{3.967}{310}{802} & \btr{0.741}{109}\\
			\hline
			p90n20 & \btc{0.015}{824}{844} & \btm{0.28}{1.71}{1,220}{3,187} & \btm{3}{4.084}{1,220}{3,185} & \btr{0.806}{402}\\
			\hline
			p90n50 & \btc{0.961}{4,932}{5,000} & \btm{2.51}{11.53}{7,550}{19,833} & \btm{3}{457.61}{7,550}{19,832} & \btr{1.129}{2,350}\\
			\hline
			p90n100 & \btc{Timeout}{19,464}{19,682} & \btm{17.34}{42.02}{30,100}{79,281} & \btm{5}{Timeout}{30,100}{79,272}& \btr{2.875}{9,215}\\
			\hline
			p90n200 & \btc{Timeout}{76,909}{77,467} & \btm{160.18}{267.17}{120,200}{316,546} & \btm{20}{Timeout}{120,200}{316,527}& \btr{278.734}{36,482}\\
			\hline
			p90n300 & \btc{Timeout}{100,051}{109,501} & \btm{567.73}{928.94}{270,300}{712,354} & \btm{34}{Timeout}{270,300}{712,324}& \btr{Timeout}{81,759}\\
			\hline
			p90n400 & \btc{Timeout}{177,147}{193,493} & Timeout & Timeout& \btr{Timeout}{144,864}\\
			\hline
			p90n500 & \btc{Timeout}{478,931}{481,977} & Timeout & Timeout & \btr{Timeout}{225,901}\\
			\hline
			\hline
			p100n10 & \btc{0.007}{131}{141} & \btm{0.1}{0.11}{310}{810} & \btm{3}{3.839}{310}{810}& \btr{0.588}{120}\\
			\hline
			p100n20 & \btc{0.013}{461}{481} & \btm{0.31}{0.35}{1,220}{3,220} & \btm{3}{4.166}{1,220}{3,220}& \btr{0.740}{440}\\
			\hline
			p100n50 & \btc{0.046}{2,651}{2,701} & \btm{2.52}{3.11}{7,550}{20,050} & \btm{3}{6.088}{7,550}{20,050}& \btr{0.955}{2,600}\\
			\hline
			p100n100 & \btc{0.141}{10,301}{104,01} & \btm{17.4}{24.19}{30,100}{80,100} & \btm{5}{14.438}{30,100}{80,100}& \btr{1.464}{10,200}\\
			\hline
			p100n200 & \btc{0.384}{40,601}{40,801} & \btm{159.96}{259.22}{120,200}{320,200} & \btm{15}{77.492}{120,200}{320,200}& \btr{2.084}{40,400}\\
			\hline
			p100n300 & \btc{0.804}{90,901}{91,201} & \btm{565.46}{921.05}{270,300}{720,300} & \btm{33}{257.52}{270,300}{720,300}& \btr{3.322}{90,600}\\
			\hline
			p100n400 & \btc{1.395}{161,201}{161,201} & Timeout & Timeout& \btr{4.772}{160,800}\\
			\hline
			p100n500 & \btc{2.175}{251,501}{252,001} & Timeout & Timeout& \btr{6.115}{251,000}\\
			\hline
			\hline
			\caption{Running Statistics on Finding Maximal Relaxed Clique (MAP Inference)} % needs to go inside longtable environment
%			\label{tab:performance2}
		\end{longtable}
	\end{center}

The table gives the results of running maximal relaxed clique with various graph instances on each of the four underlying solvers. The graph instances range from a small size of 10 nodes to a large size of 500 nodes. We compare the system based on the results of the experiment primarily on the performance of the respective solvers. Note that while \la and \lm with {\sc rockit} gives exact solutions, \lm with {\sc alchemy} and {\sc tuffy} may return sub-optimal solutions. The quality of answers for these solvers based on different parameters is discussed in the next experiment.

The naive grounding (the MRF creation time) of \al is a primary bottleneck for the solver. Even after the compact encoding based on Equation \eqref{eq:completion-tse} used in the translation for MLN solvers, it times out during grounding for $N>340$.  Solver {\sc tuffy} uses database for grounding and noticeably has better grounding times than {\sc alchemy} for most graph instances ignoring the constant time it takes for {\sc tuffy} to connect to {\sc postgres} database server. In spite of better grounding mechanisms than {\sc alchemy}, solver {\sc tuffy} times out while grounding with $N>370$ in our experiments. Although using database for optimizes the process in MLN solvers it is still not good enough when compared to the grounding process of {\sc clingo} and {\sc rockit}. Solvers {\sc clingo} uses {\sc gringo} for grounding {\sc rockit} uses {\sc mysql} in conjunction with {\sc gurobi} for grounding. Both {\sc clingo} and {\sc rockit} can ground all instances of the graph. The grounding time for {\sc alchemy} and {\sc tuffy} is comparable to solving time while it is negligible compared to solving time for {\sc clingo} and {\sc rockit}. 

Grounding time of all solvers constantly increases as the number of nodes increases. Interestingly, this increase in grounding time for bigger graph instances does not correlate with the increase in solving time. For MLN solvers {\sc alchemy} and {\sc tuffy} solving time increases constantly with graph size regardless of the sparsity of graph. A graph instance where all the nodes are connected $p = 1$ to each other, a fully connected graph, is solved much faster than all other instances by {\sc clingo} and {\sc rockit}. For \cli, the running time is sensitive to particular problem instance due to the exact optimization algorithm CDNL-OPT \cite{gebser11multi} used in {\sc clingo}. The non-deterministic nature of CDNL-OPT also brings randomness on the path through which an optimal solution is found, which makes the running time differ even among similar-sized problem instances, while in general, as the size of the graph increases, the search space gets larger, thus the running time increases. Both {\sc alchemy} and {\sc tuffy} use MaxWalkSat for MAP inference which allows the solver to return sub-optimal solutions. The approximate nature of the method allows relatively consistent running time for different problem instances, as long as parameters such as the maximum number of iterations/tries are fixed among all experiments. The running time was also not affected much by the edge probability. System {\sc rockit} uses Cutting Plane Inference (CPI) \cite{riedel2012improving} along with Cutting Plane Aggregation (CPA) \cite{noessner2013rockit} for inference. Using CPI, {\sc rockit} iteratively solves a partial version of the complete ground network based on the Cutting Plane approach. At each iteration, it checks for all the constraints that are unsatisfied and adds them to the ILP solver {\sc gurobi}. In a fully connected graph instance where all of the nodes are connected, the number of rules violated due to the last rule of the program is 0, and therefore, {\sc rockit} runs faster .

Performance wise \la outperforms \lm in a fully connected graph ($p=1$) and in all other instances \lm with {\sc rockit} clearly outperforms others. One factor that aids in {\sc rockit}'s performance is the internal solver {\sc gurobi}'s multi-core architecture. {\sc gurobi} uses all the cores available on the machine for computation while \cli, {\sc alchemy} and {\sc tuffy} all use a single core for computation. {\sc gurobi} is also the fastest commercial ILP solver according to some benchmark results\footnote{The benchmark results are available on the {\sc gurobi} website at www.gurobi.com}.
\EOCC

%-----------------------------------------------------------------------------------------------------
\subsection{Link Prediction in Biological Networks - Another Comparison with {\sc problog2} on a Real World Problem}
%-----------------------------------------------------------------------------------------------------

Public biological databases contain huge amounts of rich data, such as annotated sequences, proteins, genes and gene expressions, gene and protein interactions, scientific articles, and ontologies. Biomine \cite{eronen12biomine} is a system that integrates cross-references from several biological databases into a graph model with multiple types of edges. Edges are weighted based on their type, reliability, and informativeness.

We use graphs extracted from the Biomine network. The graphs are extracted around genes known to be connected to the Alzheimer's disease (HGNC ids 620, 582, 983, and 8744). A typical query on such a database of biological concepts is whether a given gene is connected to a given disease. In a probabilistic graph, the importance of the connection can be measured as the probability that a path exists between the two given nodes, assuming that each edge is true with the specified probability, and that edges are mutually independent \cite{sevon06link}. Nodes in the graph correspond to different concepts such as gene, protein, domain, phenotype, biological process, tissue, and edges connect related concepts. Such a program can be expressed in the language of {\sc problog2} as
\begin{lstlisting}
p(X,Y) :- drc(X,Y).
p(X,Y) :- drc(X, Z), Z \== Y, p(Z, Y).
\end{lstlisting}
The \la encoding for the same problem is
\begin{lstlisting}
p(X,Y) :- drc(X,Y).
p(X,Y) :- drc(X, Z), Z != Y, p(Z, Y).
\end{lstlisting}
The evidence file contains weighted edges \lstinline|drc/2| encoded as
\begin{lstlisting}
0.942915444848::drc('hgnc_983','pubmed_11749053').
0.492799999825::drc('pubmed_10075692','hgnc_620').
0.434774330065::drc('hgnc_620','pubmed_10460257').
...
\end{lstlisting}
The same evidence used for {\sc problog2} is processed to work with the syntax of \la as
\begin{lstlisting}
2.804443020124533 drc("hgnc_983","pubmed_11749053").
-0.028801991603851305 drc("pubmed_10075692","hgnc_620").
-0.26239795220008383 drc("hgnc_620","pubmed_10460257").
...
\end{lstlisting}
where a probability $p$ is turned into weight $ln(\frac{p}{1-p})$.
We test the systems on varying graph sizes ranging from 366 nodes, 363 edges to 3724 nodes, 23135 edges. The experiment was run on a 40 core Intel(R) Xeon(R) CPU E5-2640 v4 @ 2.40GHz machine with 128 GB of RAM. The timeout for the experiment was set to 20 minutes. 

\begin{table}\label{tab:biomin}
	\begin{longtable}{|c|c|c|c|}
	\hline
	{\bf Nodes} & {\bf Edges} & \la & {\bf {\sc problog2}} \\
	\hline
	\hline
	366 &	363 &	6.733s &	0.977s\\
	\hline
	1677 &	2086 &	279.87s &	Timeout\\
	\hline
	1982 &	4143 &	239.916s	& Timeout \\
	\hline
	2291 &	6528 &	481.177s &	Timeout\\
	\hline
	2588 &	9229 &	535.947s &	Timeout\\
	\hline
	2881 &	12248 &	892.525s &	Timeout\\
	\hline
	3168 &	15583 &	1084.231s &	Timeout\\
	\hline
	3435 &	19204 &	958.882s &	Timeout\\
	\hline
	3724 &	23135 &	Timeout &	Timeout\\
	\hline
	\end{longtable}
		\caption{{\sc problog2} vs. \la Comparison on Biomine Network}
\end{table}

We perform MAP inference for comparison. The evidence contains one single fact
\begin{verbatim}
p("hgnc_983","pubmed_11749053")
\end{verbatim}
for all the experiments. Note that {\sc problog2} only considers facts that are actually needed to ground the evidence and thus the result of MAP inference from {\sc problog2} is not a complete interpretation. On the other hand, {\sc lpmln2asp} outputs a complete interpretation.

Table \ref{tab:biomin} shows the results of the experiment. Apart from the smaller graph instances where {\sc problog2} is faster than {\sc lpmln2asp}, \la significantly outperforms {\sc problog2} for medium to large graphs for MAP inference. In fact, for graphs with nodes greater than 1677 {\sc problog2} times out. For Marginal inference, to check for the probability of path between two genes, \la times out with just 25 nodes and therefore it is infeasible to experiment for marginal probability on \la. The sampling based approach of {\sc problog2} computes the probability of a path from \lstinline|'hgnc_983'| to \lstinline|'hgnc_620'| in 13 seconds. This experiment goes on to show that for MAP inference, our implementation far outperforms the current implementation of {\sc problog2} while being significantly slower in computing Marginal and Conditional probabilities.

%-----------------------------------------------------------------------------------------------------
\subsection{Social influence of smokers - Computing MLN using \la}
%-----------------------------------------------------------------------------------------------------

%We use 
%Example \ref{ex:lpmln2sap-mln} used in Section \ref{sec:lpmln2asp2-mln} 
Following Section~\ref{ssec:mln2asp}, we compare the scalability of \la for MAP inference on MLN encodings and compare with the MLN solvers {\sc alchemy}, {\sc tuffy} and {\sc rockit} used in {\sc lpmln2mln}. We scale the example by increasing the number of people and relationships among them. 

The \la encoding of the example used in the experiment is
\begin{lstlisting}
1.1 cancer(X) :- smokes(X).
1.5 smokes(Y) :- smokes(X), influences(X, Y).
{smokes(X)} :- person(X).
{cancer(X)} :- person(X).
\end{lstlisting}
The {\sc alchemy} encoding of the example is
\begin{lstlisting}
smokes(node)
influences(node,node)
cancer(node)

1.1 smokes(x) => cancer(x)
1.5 smokes(x) ^ influences(x,y) => smokes(y)
\end{lstlisting}
and is run with the command line
\begin{lstlisting}
infer -m -i input -e evidence -r output -q cancer -ow smokes,cancer
\end{lstlisting}
The {\sc tuffy} encoding of the example is\footnote{* makes the predicate closed world assumption}
\begin{lstlisting}
smokes(node)
*influences(node,node)
cancer(node)

1.1 smokes(x) => cancer(x)
1.5 smokes(x) , influences(x,y) => smokes(y)
\end{lstlisting}
and is run with the command line
\begin{lstlisting}
java -jar tuffy.jar -i input -e evidence -r output -q cancer
\end{lstlisting}
The {\sc rockit} encoding of the example is
%\footnote{* makes the predicate closed world assumption}
\begin{lstlisting}
smokes(node)
*influences(node,node)
cancer(node)

1.1 !smokes(x) v cancer(x)
1.5 !smokes(x) v !influences(x,y) v smokes(y)
\end{lstlisting}
and is run with the command line
\begin{lstlisting}
java -jar rockit.jar -input input -data evidence -output output
\end{lstlisting}

The data was generated such that for each person $p$, the person \emph{smokes} with an 80\% probability, and $p$ \emph{influences} every other person with a 60\% probability. We generate evidence instances based on different number of persons ranging from 10 to 1000. We compare the performance of the solvers based on the time it takes to compute the MAP estimate. The experiment was run on a 40 core Intel(R) Xeon(R) CPU E5-2640 v4 @ 2.40GHz machine with 128 GB of RAM. The timeout for the experiment was set to 20 minutes. 

\begin{center}
	\begin{longtable}{|c|c|c|c|c|}
		\hline
		\# of Persons & \la {\bf w.} {\sc clingo} {\bf 4.5} & \al \ {\bf 2.0}& \tu \ {\bf 0.3}& \ro \ {\bf 0.5}\\
		\hline
		\hline
		10 & 0 & 0.04 & 1.014 & 0.465 \\
		\hline
		50 & 0.03 & 1.35 & 1.525 & 0.676 \\
		\hline
		100 & 0.10 & 18.87 & 1.560 & 0.931 \\
		\hline
		200 & 0.32 & 435.71 & 2.672 & 1.196 \\
		\hline
		300 & 0.7 & Timeout & 4.054 & 1.660 \\
		\hline
		400 & 1.070 & Timeout & 4.505 & 1.914 \\
		\hline
		500 & 1.730 & Timeout & 5.935 & 2.380 \\
		\hline
		600 & 2.760 & Timeout & 7.683 & 2.822 \\
		\hline
		700 & 3.560 & Timeout & 10.390 & 3.274 \\
		\hline
		800 & 4.72 & Timeout & 11.384 & 3.727 \\
		\hline
		900 & Timeout & Timeout & 12.056 & 4.012 \\
		\hline
		1000 & Timeout & Timeout & 12.958 & 4.678 \\
		\hline
		\hline
		\caption{Performance of solvers on MLN program}
		\label{tab:performance-mln}
	\end{longtable}
\end{center}

Table \ref{tab:performance-mln} lists the computation time in seconds for each of the four solvers on instances of domains of varying size. \la is the best performer for the number of people till 600. {\sc alchemy} is the worst performer out of all 4 and for instances with number of people greater than 200 it times out. As expected, for {\sc alchemy}, grounding is the major bottleneck. For the instance with 200 persons, {\sc alchemy} grounds it in 422.85 seconds and only takes 9 seconds to compute the MAP estimate. 
{\sc tuffy} and {\sc rockit} have more scalable grounding times.
%Since  grounding has been addressed in {\sc tuffy} and {\sc rockit}, these solvers are able to scale way better than {\sc alchemy}. 
{\sc rockit} has the best results amongst all the solvers. This experiment shows that for medium sized instances, our implementation is comparable to  the fastest available solver for MAP inference on MLN programs.

% which is surprising since they both use RDBMS for grounding and computation in {\sc rockit} should be faster than {\sc tuffy} because of {\sc gurobi}. 
%This experiment shows that for medium sized instances, our implementation is the fastest amongst all the available solvers for MAP inference on MLN programs.

%\bibliography{bib,bib2}

\end{appendix}